\documentclass[twoside]{article}

\usepackage[preprint]{aistats2027}
\usepackage{amsmath,amssymb,amsthm}
\usepackage{booktabs}
\usepackage{array}
\usepackage{graphicx}
\usepackage{microtype}
\usepackage[round]{natbib}
\usepackage{url}

\graphicspath{{figures/}}

\newtheorem{lemma}{Lemma}
\newtheorem{theorem}{Theorem}
\newtheorem{proposition}{Proposition}
\newtheorem{corollary}{Corollary}
\newtheorem{remark}{Remark}

\newcommand{\Rb}{\mathbb{R}}
\newcommand{\Eb}{\mathbb{E}}
\newcommand{\NPDF}{\mathcal{N}}
\newcommand{\bH}{\mathbf{H}}
\newcommand{\bG}{\mathbf{G}}
\newcommand{\bS}{\mathbf{S}}
\newcommand{\bK}{\mathbf{K}}
\newcommand{\bPhi}{\mathbf{\Phi}}
\newcommand{\bSig}{\mathbf{\Sigma}}
\newcommand{\Gt}{\widetilde{\bG}}
\newcommand{\yt}{\widetilde{y}}
\newcommand{\Rsp}{\mathcal{R}}
\newcommand{\dsr}{d_{\mathrm{search}}}

\providecommand{\const}[1]{#1}  

\newcommand{\vBudFullDelta}{+0.54}
\newcommand{\vBudFullP}{0.97}

\newcommand{\vBudHi}{96}
\newcommand{\vBudLo}{6}
\newcommand{\vBudRatio}{16}

\newcommand{\vBudRowDelta}{-0.13}
\newcommand{\vBudRowP}{0.97}

\newcommand{\vFSGrowthConfigs}{54}
\newcommand{\vFSGrowthSlack}{-3.3e-16}

\newcommand{\vFSNLargeN}{25600}
\newcommand{\vFSNSmallN}{400}
\newcommand{\vFSNoClipCells}{536}
\newcommand{\vFSNoClipFailP}{64}
\newcommand{\vFSNoClipFull}{1.99}

\newcommand{\vFSNoClipRow}{0.79}

\newcommand{\vFSNoClipRowMed}{0.49}

\newcommand{\vFSSlopeFull}{-0.71}
\newcommand{\vFSSlopeFullHi}{-0.63}
\newcommand{\vFSSlopeFullInP}{+1.03}
\newcommand{\vFSSlopeFullInPHi}{+1.19}
\newcommand{\vFSSlopeFullInPLargeN}{+0.73}
\newcommand{\vFSSlopeFullInPLo}{+0.85}
\newcommand{\vFSSlopeFullInPSmallN}{+1.33}
\newcommand{\vFSSlopeFullLo}{-0.78}
\newcommand{\vFSSlopeRow}{-0.50}
\newcommand{\vFSSlopeRowHi}{-0.48}
\newcommand{\vFSSlopeRowInP}{-0.02}
\newcommand{\vFSSlopeRowInPHi}{+0.10}
\newcommand{\vFSSlopeRowInPLo}{-0.14}
\newcommand{\vFSSlopeRowLo}{-0.52}

\newcommand{\vLBBandAboveHi}{11\%}
\newcommand{\vLBBandAboveLo}{0\%}
\newcommand{\vLBBandAboveMin}{2.22}
\newcommand{\vLBBandAboveRate}{0\%}
\newcommand{\vLBBandBelowCells}{13}
\newcommand{\vLBBandBelowDraws}{104}
\newcommand{\vLBBandBelowHi}{100\%}
\newcommand{\vLBBandBelowLo}{97\%}
\newcommand{\vLBBandBelowMax}{1.07}
\newcommand{\vLBBandBelowRate}{100\%}
\newcommand{\vLBBandNearCells}{2}

\newcommand{\vLBBandNearHi}{96\%}
\newcommand{\vLBBandNearLo}{54\%}
\newcommand{\vLBBandNearRate}{81\%}

\newcommand{\vLBDepthAllExpM}{+0.41}
\newcommand{\vLBDepthAllExpN}{-0.38}

\newcommand{\vLBDepthCells}{224}
\newcommand{\vLBDepthSeeds}{8}

\newcommand{\vLBFitC}{6.36}
\newcommand{\vLBFitCHi}{6.49}
\newcommand{\vLBFitCLo}{6.23}

\newcommand{\vLBFitCells}{155}
\newcommand{\vLBFitDropped}{69}
\newcommand{\vLBFitExpM}{+0.50}
\newcommand{\vLBFitExpMHi}{+0.51}
\newcommand{\vLBFitExpMLo}{+0.48}
\newcommand{\vLBFitExpN}{-0.44}
\newcommand{\vLBFitExpNHi}{-0.43}
\newcommand{\vLBFitExpNLo}{-0.46}

\newcommand{\vLBFitThresh}{1.0}

\newcommand{\vLBLOneCells}{16}
\newcommand{\vLBLOneFloorGap}{4.4e-16}
\newcommand{\vLBLOneGapOut}{1.81}
\newcommand{\vLBLOneLeak}{4.5e-16}
\newcommand{\vLBLogCorrHi}{-0.42}
\newcommand{\vLBLogCorrLo}{-0.45}
\newcommand{\vLBOutlierMax}{88}
\newcommand{\vLBOutlierP}{256}

\newcommand{\vLabCells}{16}

\newcommand{\vOptCells}{64}

\newcommand{\vOptGapMax}{2.7e-12}
\newcommand{\vOptGapMed}{1.4e-14}

\newcommand{\vProtNstarRestarts}{5}

\newcommand{\vTwoSlopeN}{-0.50}
\newcommand{\vTwoSlopeNHi}{-0.48}
\newcommand{\vTwoSlopeNLo}{-0.52}
\newcommand{\vTwoSlopeP}{-0.02}
\newcommand{\vTwoSlopePHi}{+0.10}
\newcommand{\vTwoSlopePLo}{-0.14}

\newcommand{\vkAnisoCalExp}{+0.203}
\newcommand{\vkAnisoCalExpHi}{+0.283}
\newcommand{\vkAnisoCalExpLo}{+0.115}
\newcommand{\vkAnisoCalRTwo}{0.63}
\newcommand{\vkAnisoCells}{16}

\newcommand{\vkAnisoShiftHi}{1.76}
\newcommand{\vkAnisoShiftLo}{0.95}
\newcommand{\vkAnisoShiftMed}{1.22}
\newcommand{\vkAnisoShiftRFour}{1.09}
\newcommand{\vkAnisoShiftRSixteen}{1.60}

\newcommand{\vkBayesPop}{0.237137}

\newcommand{\vkCritCalHi}{20.5}
\newcommand{\vkCritCalLo}{15.4}
\newcommand{\vkCrossFifteen}{0.60}

\newcommand{\vkCrossTen}{1.20}
\newcommand{\vkCrossTwenty}{0.30}

\newcommand{\vkDepthRsqLogLogM}{0.9569}
\newcommand{\vkDepthRsqLogM}{0.9589}

\newcommand{\vkExpCritFifteen}{0.150}
\newcommand{\vkExpCritTen}{0.142}
\newcommand{\vkExpCritTwenty}{0.145}

\newcommand{\vkFormLinear}{0.8082}
\newcommand{\vkFormLog}{0.8016}

\newcommand{\vkFormPower}{0.8127}

\newcommand{\vkKLPredicted}{0.1563}
\newcommand{\vkKLRatio}{0.1560}
\newcommand{\vkKappaStar}{0.229333}

\newcommand{\vkLawCalCells}{11}
\newcommand{\vkLawCalExp}{0.156}
\newcommand{\vkLawCalHi}{0.174}
\newcommand{\vkLawCalLo}{0.118}
\newcommand{\vkLawCalRsq}{0.965}

\newcommand{\vkLawExpHi}{0.166}
\newcommand{\vkLawExpLo}{0.110}
\newcommand{\vkLawFixExp}{0.150}

\newcommand{\vkLawRsq}{0.964}

\newcommand{\vkNullAniHi}{4.6}
\newcommand{\vkNullAniLo}{3.9}
\newcommand{\vkNullIsoHi}{24.8}
\newcommand{\vkNullIsoLo}{23.0}

\newcommand{\vkNullRatio}{5}

\newcommand{\vkRankCells}{28}
\newcommand{\vkRankExp}{0.325}
\newcommand{\vkRankHi}{0.399}
\newcommand{\vkRankLo}{0.259}
\newcommand{\vkRankPca}{-0.031}

\newcommand{\vkTaskACross}{0.006}

\begin{document}

\runningtitle{Search Dimension in Unlabeled Projection Pursuit}
\runningauthor{Grozavescu, Girolami}
\twocolumn[
\aistatstitle{Search Dimension in Unlabeled Projection Pursuit:\\
A Scaling Law for Subspace Restriction}
\aistatsauthor{Rare\textcommabelow{s} Grozavescu \And Mark Girolami}
\aistatsaddress{University of Cambridge\\ \texttt{rg625@cam.ac.uk}
  \And University of Cambridge}
]


\begin{abstract}
Projection pursuit searches for a direction along which the data look least Gaussian. When the observation space contains a large Gaussian complement, the empirical objective can be minimized by a direction that carries no signal, with empirical kurtosis as low as at the truth. Sample splitting exposes rather than repairs this failure. Appending coordinates independent of the latent regime degrades the search while leaving Bayes recoverability unchanged. Restricting the search to the column space of a known forward operator removes the failure exactly on the negative-kurtosis branch. Estimating a principal subspace from the data is the alternative. In a controlled two-component model, the leading sufficient scalings differ in the gain with which the operator transmits the discriminant: $\varsigma^{-4}$ for covariance-spike estimation and $\varsigma^{-8}$ for fourth-moment search. At fixed search dimension, the measured threshold ratio collapses onto $n/p^2$ with exponent $\vkLawCalExp$, close to the predicted $1/8$. This is an empirically supported scaling motivated by sufficient bounds, not a proved asymptotically tight law. When the search dimension is varied, the measured exponent is $\vkRankExp$, substantially larger than $1/8$, and the tested range does not identify its functional form. The crossing location also depends on calibration and model configuration. Under a downstream excess-error criterion, the scaling largely disappears.
\end{abstract}

\section{Introduction}
\label{sec:intro}

Projection pursuit searches for a direction along which the data depart from a reference distribution. In high dimensions, this creates a basic statistical problem: the search is performed over many directions, most of which may carry no information about the latent structure. An empirical search can therefore find a projection that looks non-Gaussian even when the direction itself is unrelated to the signal.

We study this problem for minimum-kurtosis projection pursuit in a linear--Gaussian inverse model. After whitening, the observations have the form
\[
\yt=\Gt z+\tilde\varepsilon,
\qquad
\tilde\varepsilon\sim\NPDF(0,I_p).
\]
The signal $\Gt z$ lies in the column space
\[
\Rsp=\operatorname{col}(\Gt),
\]
while the orthogonal complement is independent standard Gaussian noise. Thus the observation model identifies a natural search space, but an unrestricted empirical search can exploit the Gaussian complement to obtain a spuriously low sample kurtosis.

Our first result makes this failure exact. When the Gaussian complement has dimension at least $n-1$ for even $n$, the empirical excess kurtosis reaches its universal lower bound $-2$ on an uninformative direction in $\Rsp^\perp$, while the population kurtosis of that direction is zero. Before this exact degeneracy is reached, the complement already contains spurious minima whose depth grows with the searched dimension. The failure is therefore a property of the empirical search rather than a loss of information about the latent regime.

The population problem points to a simple remedy. If the forward operator is known, restricting the search to $\Rsp$ does not change the population optimum on the negative-kurtosis branch. For a direction mixing an informative component with an orthogonal Gaussian component,
\[
\kappa(u)=\kappa_a\lambda(\theta)^2,
\]
so the Gaussian complement can only attenuate the magnitude of the population excess kurtosis. Restricting the search to $\Rsp$ is therefore population-lossless on this branch while removing irrelevant directions from the empirical optimization problem.

When $\Rsp$ is unknown, the problem separates into two tasks: estimating the signal subspace and finding the direction within it.. This creates two separate statistical tasks: discovering the signal subspace and finding the direction inside it. We compare them in a controlled two-component model. If $\varsigma$ is the gain with which the forward operator transmits the latent mean-gap direction, principal-subspace estimation is governed by a covariance spike of order $\varsigma^2$, while the fourth-moment signal in the restricted search is of order $\varsigma^4$ at weak gain. The leading sufficient scalings are
\[
\varsigma_{\mathrm{PCA}}\propto(p/n)^{1/4},
\qquad
\varsigma_{\mathrm{KPP}}\propto(\dsr/n)^{1/8},
\]
which give
\[
\frac{\varsigma_{\mathrm{KPP}}}
     {\varsigma_{\mathrm{PCA}}}
\propto
\left(\frac{n}{p^2}\right)^{1/8}
\]
at fixed search dimension.

We test this comparison rather than treating the sufficient bound as a tight law. At fixed search dimension, the measured exponent is $\vkLawCalExp$ (95\% CI $[\vkLawCalLo,\vkLawCalHi]$), close to $1/8$. When the search dimension is varied, the measured exponent is $\vkRankExp$, substantially larger than $1/8$, and the tested range does not distinguish power-law, logarithmic, and linear forms. The threshold crossing also changes with the angular criterion, covariance-spike ordering, and latent covariance.

The paper has three main contributions:

\begin{itemize}\itemsep2pt

\item We characterize a high-dimensional failure mode of unrestricted minimum-kurtosis projection pursuit, including an exact empirical degeneracy and a general spurious-minimum bound, and show that restriction to the known signal subspace removes the population failure on the negative-kurtosis branch.

\item We derive a finite-sample direction-recovery bound whose dependence on the searched dimension motivates a comparison between exact-subspace search and data-driven subspace estimation.

\item We measure this comparison and its limits. At fixed search dimension the threshold ratio follows the predicted $1/8$ gain exponent, while the observed search-dimension dependence is different. We also show that the crossing depends on calibration and model configuration and that sample splitting cannot repair a search-limited procedure.

\end{itemize}

The comparison is about direction recovery. When thresholds are instead defined through downstream excess error, the scaling largely disappears because prediction error is less sensitive to weak-signal direction differences.

\section{Setup and exact reference}
\label{sec:setup}

We use a linear--Gaussian inverse model in which the latent variable has a finite Gaussian-mixture prior:
\begin{equation}
\begin{gathered}
z \sim
\textstyle\sum_{k=1}^{M}\pi_k\,\NPDF(\mu_k,\bSig_k),
\\
a\mid z\sim\NPDF(\bPhi z,\bK),
\qquad
y=\bH a+\varepsilon,
\end{gathered}
\label{eq:model}
\end{equation}
with $\varepsilon\sim\NPDF(0,\sigma^2I)$.

Writing $w=a-\bPhi z$, setting $\bG=\bH\bPhi$, and marginalizing $w$ gives
\[
y\mid z\sim\NPDF(\bG z,\bS),
\qquad
\bS=\bH\bK\bH^\top+\sigma^2I.
\]

Whitening by $\bS$ gives
\begin{equation}
\begin{gathered}
\yt=\bS^{-1/2}y
=\Gt z+\tilde\varepsilon,
\\
\Gt=\bS^{-1/2}\bG,
\qquad
\tilde\varepsilon\sim\NPDF(0,I_p).
\end{gathered}
\label{eq:whitened}
\end{equation}

We call
\[
\Rsp=\operatorname{col}(\Gt)
\]
the signal subspace and write $r=\dim(\Rsp)$.

\paragraph{Recoverability reference.}
For the Gaussian-mixture model, the class-conditional distribution of $y$ is Gaussian with mean $\bG\mu_k$ and covariance
\[
\bG\bSig_k\bG^\top+\bS.
\]
The corresponding Bayes classifier is
\[
\arg\max_k\;
\pi_k
\NPDF
\left(
y;
\bG\mu_k,
\bG\bSig_k\bG^\top+\bS
\right).
\]
In the experiments this exact classifier is used as a reference. The downstream procedures are evaluated on held-out data and do not use held-out regime labels during fitting. Appendix~\ref{app:info} lists the information available to each method, and Appendix~\ref{app:details} gives the matching and evaluation protocol.

\paragraph{Projection pursuit.}
For two Gaussian components with weights $p,q$, projected separation $\Delta$, and projected variances $v_1,v_2$, the population excess kurtosis is
\begin{equation}
\kappa
=
\frac{pq\,N}
     {(pv_1+qv_2+pq\Delta^2)^2},
\label{eq:kappa}
\end{equation}
where
\[
N=
3(v_1-v_2)^2
+
6\Delta^2(q-p)(v_1-v_2)
+
\Delta^4(1-6pq).
\]
The main experiments use two equally weighted components with isotropic within-component covariance and a mean separation along one latent coordinate.

\section{The operator determines the population search space}
\label{sec:mech}

Let $P$ be the orthogonal projector onto $\Rsp$.

\begin{lemma}[Noise-only complement]
\label{lem:complement}
For the whitened model in \eqref{eq:whitened},
\[
(I-P)\yt=(I-P)\tilde\varepsilon.
\]
The complement is therefore standard Gaussian on $\Rsp^\perp$ and is independent of $z$ and of $P\yt$.
\end{lemma}

The lemma identifies the central geometric feature of the problem: directions in $\Rsp^\perp$ contain no latent information at all.

\begin{lemma}[Exact dilution]
\label{lem:dilution}
Let
\[
u=\cos\theta\,a+\sin\theta\,b,
\]
where $a\in\Rsp$ and $b\in\Rsp^\perp$ are unit vectors. If
\[
v=\operatorname{Var}(a^\top\yt)
\]
and $\kappa_a$ is the excess kurtosis of $a^\top\yt$, then
\begin{equation}
\kappa(u)
=
\kappa_a\,\lambda(\theta)^2,
\qquad
\lambda(\theta)
=
\frac{v\cos^2\theta}
     {v\cos^2\theta+\sin^2\theta}.
\label{eq:dilution}
\end{equation}
Thus $\lambda(\theta)\in[0,1]$, with equality to one exactly when $\sin\theta=0$.
\end{lemma}

The complement therefore attenuates the magnitude of excess kurtosis. When the informative population direction has negative excess kurtosis, adding an orthogonal Gaussian component cannot improve the population objective. Restricting the search to $\Rsp$ is consequently lossless for the population minimization problem on the negative-kurtosis branch.

If the relevant in-subspace kurtosis is positive, the minimum is instead attained in the Gaussian complement. Our direction-recovery experiments are on the negative-kurtosis branch.

The same dilution mechanism extends beyond equal mixture weights; the resulting population amplitude vanishes at specific weight and covariance-separation boundaries. These scope conditions are given in Appendix~\ref{app:population}. The main experiments remain in the negative-kurtosis regime.

\section{Finite-sample search and spurious minima}
\label{sec:theory}

Let $\mathcal S$ be the subspace searched by the empirical procedure, with
\[
d_{\mathrm{search}}
=
\dim(\mathcal S).
\]
The finite-sample question is how uniformly the empirical kurtosis approximates its population value over the searched directions.

\begin{lemma}[Variance floor and sub-Gaussianity]
\label{lem:floor}
For every unit $u$,
\[
\operatorname{Var}(u^\top\yt)
=
\operatorname{Var}(u^\top\Gt z)+1
\ge1,
\]
and
$u^\top(\yt-\Eb\yt)$ is sub-Gaussian with a scale $\sigma^2$ bounded by an absolute constant times
\[
1+\|\Gt\|_{\mathrm{op}}^2
\left[
\max_k\|\bSig_k\|_{\mathrm{op}}
+
\max_k\|\mu_k-\bar\mu\|^2
\right].
\]
\end{lemma}

The bound does not depend on the ambient dimension $p$, so the sub-Gaussian scale remains fixed across the ambient-dimension sweeps.

\paragraph{Uniform concentration.}
Let
\[
\varphi_T(x)
=
\operatorname{sign}(x)\min(|x|,T)
\]
and
\[
z_i=\yt_i-\bar y_n.
\]
Define
\begin{equation}
\begin{gathered}
\hat\kappa_T(u)
=
\frac{\hat m_{4,T}(u)}
     {\hat m_{2,T}(u)^2}
-3,
\\
\hat m_{q,T}(u)
=
\frac1n\sum_i
\varphi_T(u^\top z_i)^q.
\end{gathered}
\label{eq:trunc}
\end{equation}
The empirical criterion is defined on
\[
\mathcal D_T
=
\left\{
u\in S^{p-1}:
\hat m_{2,T}(u)>0
\right\},
\]
and empirical minimization is understood to be over $\mathcal D_T$.

\begin{theorem}[Uniform concentration over a searched subspace]
\label{thm:conc}
Let $\mathcal S\subseteq\Rb^p$ be a linear subspace with
$\dim(\mathcal S)=d_{\mathrm{search}}$, and let
$U\subseteq\mathcal S\cap S^{p-1}$.

Under the assumptions of Lemma~\ref{lem:floor}, take
\[
T=\sigma\sqrt{8\log n},
\qquad
L_n=d_{\mathrm{search}}\log(3n)+\log(12/\delta).
\]
There are absolute constants $C,C'$ such that if
\[
n\ge
C'\sigma^8(\log n)^3L_n,
\]
then with probability at least $1-\delta$,
\begin{equation}
\begin{aligned}
\sup_{u\in U}
\left|
\hat\kappa_T(u)-\kappa(u)
\right|
&\le
C\sigma^{10}
\Bigg[
(\log n)^{3/2}
\sqrt{\frac{L_n}{n}}
\\
&\qquad\qquad
+
(\log n)^2\frac{L_n}{n}
\Bigg].
\end{aligned}
\label{eq:conc}
\end{equation}
\end{theorem}

The important feature for the search problem is the appearance of $d_{\mathrm{search}}$ in $L_n$. An unrestricted search uses $d_{\mathrm{search}}=p$, whereas exact restriction to $\Rsp$ uses $d_{\mathrm{search}}=r$.

The experiments optimize the untruncated criterion. A sample-checkable no-clipping certificate implies that the truncated and untruncated criteria coincide on the searched set. We evaluate this certificate for every configuration and use the truncated bound above. Without the certificate, the corresponding untruncated result has a second-order $L_n^2/n$ term; the full statement is given in Appendix~\ref{app:fs}.

\begin{lemma}[Dilution identity and global quadratic growth]
\label{lem:growth}
Take $\Gt$ with orthonormal columns,
\[
\bSig_1=\bSig_2=s^2I_r,
\quad
\mu_1-\mu_2=\Delta e_1,
\quad
\pi_1=\pi_2=1/2,
\]
and put
\[
V=1+s^2+\Delta^2/4,
\qquad
\kappa_\star=-\frac{\Delta^4}{8V^2}.
\]
With
\[
g=\Gt^\top u,
\qquad
t=g_1^2,
\qquad
w=\|g\|^2,
\]
\eqref{eq:kappa} becomes
\[
\kappa(u)
=
-\frac{\Delta^4t^2}
       {8(1+s^2w+\Delta^2t/4)^2}.
\]
The population minimizers are $\pm\Gt e_1$ and
\begin{equation}
\kappa(u)-\kappa_\star
\ge
|\kappa_\star|
\min(1,2/V)\,
\sin^2 d(u,u_\star),
\label{eq:growth}
\end{equation}
for every unit $u$, where
\[
d(u,u_\star)
=
\arccos|\langle u,u_\star\rangle|.
\]
\end{lemma}

For $V\ge2$, the constant
$2|\kappa_\star|/V$ is the local coefficient of $\sin^2d$ at $u_\star$, so the bound is sharp to first order.

\begin{corollary}[Sufficient sample size for direction recovery]
\label{cor:nstar}
In the setting of Theorem~\ref{thm:conc} and Lemma~\ref{lem:growth}, let $\hat u$ be any global minimizer of $\hat\kappa_T$ over $U$ and suppose $u_\star\in U$. On the event of Theorem~\ref{thm:conc},
\[
\sin^2 d(\hat u,u_\star)
\le
\frac{2\eta}
     {|\kappa_\star|\min(1,2/V)},
\]
where $\eta$ is the right-hand side of \eqref{eq:conc}.

In particular, for $V\ge2$, $d(\hat u,u_\star)\le\theta_0$ is ensured by
\begin{equation}
n\ge
C\,\sigma^{20}V^2
\frac{
d_{\mathrm{search}}\log(3n)(\log n)^3
}{
\kappa_\star^2\sin^4\theta_0
}.
\label{eq:nstar}
\end{equation}
\end{corollary}

The leading term therefore gives the sufficient scaling
\[
n\gtrsim
\frac{d_{\mathrm{search}}}{\kappa_\star^2}
\]
up to logarithmic and problem-dependent factors.

The untruncated bound has a second-order branch with stronger dependence on $d_{\mathrm{search}}$; the full comparison is in Appendix~\ref{app:fs}. The measured thresholds lie in the regime where the leading branch is the relevant one.

\subsection{Two failure modes of the unrestricted search}
\label{sec:lower}

The unrestricted search fails in two related ways.

First, the empirical excess kurtosis is bounded below by $-2$. When the Gaussian complement is sufficiently large, that floor can be attained by a completely uninformative direction. Proposition~\ref{prop:degen} makes this exact: for even $n$ and
\[
p-r\ge n-1,
\]
with probability one there is a unit
$b\in\Rsp^\perp$ such that
\[
\hat\kappa(b)=-2,
\qquad
\kappa(b)=0,
\qquad
d(b,u_\star)=\pi/2.
\]
When $r\le n-2$, no direction in $\Rsp$ attains the empirical floor almost surely.

Second, the complement contains spurious minima even before the exact floor is attainable. Proposition~\ref{prop:depth} shows that, under its stated conditions, the minimum over $\Rsp^\perp$ is at most
\[
-c\sqrt{
\frac{\log(m\wedge n^a)}{n}
}
+
C\frac{\log(m/\delta)}{n},
\qquad
m=p-r,
\]
with high probability.

The exact degeneracy is therefore the sharpest form of a broader search-space problem: the empirical objective increasingly favors directions that contain no signal as irrelevant dimensions are added.

\paragraph{Why use a fourth moment?}
The second-moment route is cheaper in the mean-separated, well-conditioned setting used for the scaling comparison. Fourth-moment pursuit remains useful when separation is expressed through covariance rather than the mean; this regime is outside the main scaling experiment.

\section{Comparing exact-subspace search with subspace estimation}
\label{sec:law}

We now compare two procedures:

\begin{enumerate}
\item \textbf{Exact-subspace KPP:} search for the minimum empirical kurtosis over the true signal subspace $\Rsp$.
\item \textbf{PCA subspace + search:} estimate the signal subspace from the sample covariance and then search within the estimated subspace.
\end{enumerate}

The first is an oracle benchmark that isolates the cost of finding the direction once the search space is known. It is not an end-to-end data-driven procedure. The second includes the cost of discovering that search space.

Let $\varsigma$ be the gain with which $\Gt$ transmits the latent mean-gap axis, so that the observable discriminant is $\varsigma$ times a unit left singular vector. For the scaling calculation, the gain is applied to that singular direction while the remaining signal directions are held fixed.

In the isotropic two-component model, the discriminant contributes a covariance spike of size
\[
\varsigma^2(s^2+\Delta^2/4).
\]
Resolving this spike against the Marchenko--Pastur bulk gives the threshold
\[
n\gtrsim p/\varsigma^4
\]
up to model-dependent constants. Thus
\[
\varsigma_{\mathrm{PCA}}
\propto
(p/n)^{1/4}.
\]

For exact-subspace KPP, define
\[
V_\varsigma
=
1+\varsigma^2(s^2+\Delta^2/4).
\]
The population kurtosis of the discriminant is
\begin{equation}
|\kappa_\star(\varsigma)|
=
\frac{\Delta^4\varsigma^4}
     {8V_\varsigma^2}.
\label{eq:kappasigma}
\end{equation}
Combining \eqref{eq:kappasigma} with Corollary~\ref{cor:nstar} gives
\[
n
\gtrsim
d_{\mathrm{search}}\,
\frac{V_\varsigma^6}{\varsigma^8}
\]
up to logarithmic and problem-dependent factors. Hence, in the weak-gain regime,
\[
\varsigma_{\mathrm{KPP}}
\propto
(d_{\mathrm{search}}/n)^{1/8}
\]
to leading order. At fixed search dimension,
\begin{equation}
\frac{\varsigma_{\mathrm{KPP}}}
     {\varsigma_{\mathrm{PCA}}}
\propto
\left(\frac{n}{p^2}\right)^{1/8}.
\label{eq:law}
\end{equation}

This comparison has a specific scope. It assumes two components, isotropic latent covariance, and that the discriminant is the covariance spike that determines the PCA threshold. If a nuisance direction has a larger spike, PCA can recover that direction before the discriminant. If the components separate only through covariance, the second moment does not provide the competing signal. We therefore interpret \eqref{eq:law} only in the mean-separated, well-conditioned regime.

\paragraph{Empirical threshold comparison.}
We define $\varsigma_{\mathrm{PCA}}(n)$ as the smallest gain on the tested gain grid at which the leakage
\[
\arcsin\|(I-\hat P)u_\star\|
\]
falls below the specified angular tolerance.

We define $\varsigma_{\mathrm{KPP}}(n)$ as the smallest gain on the same grid at which the exact-subspace KPP search recovers $u_\star$ to the same angular tolerance.

The two thresholds are therefore defined in terms of direction recovery. Thresholds that are not attained within the tested gain range are treated according to the censoring procedure described in Appendix~\ref{app:details}.

At fixed search dimension, the measured ratio follows a power law in $n/p^2$. Across $\vkLawCalCells$ configurations spanning five orders of magnitude in $n/p^2$, the calibrated fit gives exponent
\[
\vkLawCalExp
\qquad
[\vkLawCalLo,\vkLawCalHi]
\]
with
\[
R^2=\vkLawCalRsq.
\]
The confidence interval contains the leading-order predicted $1/8$ exponent.

The two thresholds are measured separately rather than by fitting a two-parameter curve directly to their ratio. The KPP threshold depends on the dimension of the searched space, while the PCA threshold depends on the ambient dimension through covariance estimation. This separation lets us test the two dependencies independently.

\begin{table}[t]
\centering
\small
\setlength{\tabcolsep}{4pt}
\begin{tabular}{@{}lrcr@{}}
\toprule
& exponent & 95\% CI & $R^2$ \\
\midrule
\multicolumn{4}{@{}l}{\emph{Collapse on $n/p^2$ at fixed $\dsr$}}\\
fixed $\const{15}^{\circ}$ criterion
&
$\vkLawFixExp$
&
$[\vkLawExpLo,\vkLawExpHi]$
&
$\vkLawRsq$
\\
calibrated criterion
&
$\vkLawCalExp$
&
$[\vkLawCalLo,\vkLawCalHi]$
&
$\vkLawCalRsq$
\\
\addlinespace[2pt]
\multicolumn{4}{@{}l}{\emph{Rank dependence, $\dsr\in\{2,\ldots,32\}$}}\\
$\varsigma_{\mathrm{KPP}}$
&
$\vkRankExp$
&
$[\vkRankLo,\vkRankHi]$
&
\\
$\varsigma_{\mathrm{PCA}}$
&
$\vkRankPca$
&
&
\\
\addlinespace[2pt]
\multicolumn{4}{@{}l}{\emph{Candidate rank forms, $R^2$}}\\
power, $\log\dsr$, linear
&
\multicolumn{3}{r}{
$\vkFormPower$,
$\vkFormLog$,
$\vkFormLinear$
}
\\
no rank dependence
&
\multicolumn{3}{r}{$0$}
\\
\bottomrule
\end{tabular}
\caption{Measured threshold scalings. The fixed-rank collapse is fitted over $\vkLawCalCells$ configurations and the rank comparison over $\vkRankCells$. The calibrated criterion is $\vkCritCalLo^\circ$--$\vkCritCalHi^\circ$ on the collapse grid. The tested rank range does not distinguish the candidate functional forms.}
\label{tab:fits}
\end{table}

\begin{table}[t]
\centering
\small
\begin{tabular}{rrrrrr}
\toprule
$p$ & $n$ & $n/p^2$ & $\varsigma_{\mathrm{PCA}}$ & $\varsigma_{\mathrm{KPP}}$ & ratio \\
\midrule
1024 & 4096 & 0.004 & 2.200 & 0.924 & 0.42 \\
512 & 4096 & 0.016 & 1.687 & 0.868 & 0.51 \\
256 & 4096 & 0.062 & 1.240 & 0.911 & 0.73 \\
128 & 4096 & 0.250 & 1.022 & 1.098 & 1.07 \\
128 & 16384 & 1.000 & 0.659 & 0.819 & 1.24 \\
32 & 4096 & 4.000 & 0.676 & 0.908 & 1.34 \\
128 & 65536 & 4.000 & 0.445 & 0.621 & 1.40 \\
32 & 16384 & 16.000 & 0.449 & 0.711 & 1.58 \\
128 & 262144 & 16.000 & 0.313 & 0.490 & 1.57 \\
32 & 65536 & 64.000 & 0.309 & 0.610 & 1.98 \\
32 & 262144 & 256.000 & 0.219 & 0.493 & 2.26 \\
\bottomrule
\end{tabular}

\caption{Threshold discriminant gains and their ratio. The ratio crosses one at
$n\approx\vkCrossFifteen\,p^2$ under the \const{15}$^\circ$ criterion at which it was measured;
Section~\ref{sec:notaconstant} shows why that number is not transportable. The same thresholds
computed on cells with no optimizer gap, and re-searched at a \const{512}-fold larger restart
budget, are in Appendix~\ref{app:restated} and agree to within the bootstrap interval of the
crossing.}
\label{tab:thresholds}
\end{table}

\begin{figure}[t]
\centering
\includegraphics[width=\columnwidth]{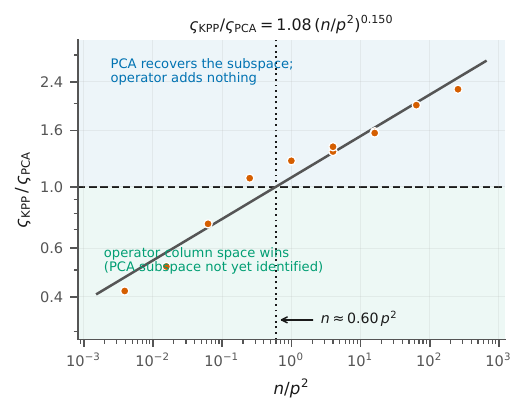}
\caption{Ratio of the exact-subspace KPP threshold to the PCA-subspace threshold against $n/p^2$ at fixed $r$. The fitted fixed-rank power law is shown for reference. The ratio crossing depends on the success criterion and model configuration; the marked value is for the baseline $\const{15}^{\circ}$ criterion and is not interpreted as a universal constant.}
\label{fig:law}
\end{figure}

\paragraph{The search-dimension dependence is different.}
Sweeping
\[
\dsr\in\{2,4,8,16,32\}
\]
at matched $(p,n)$ gives a rank exponent of $\vkRankExp$ for $\varsigma_{\mathrm{KPP}}$, compared with the predicted leading-order exponent $1/8$. The PCA threshold shows no corresponding rank dependence.

The rank effect is not explained by changing the angular criterion. Calibrating the criterion against the null still leaves the exponent well above $1/8$.

Over the tested sixteenfold range, power-law, logarithmic, and linear forms fit the rank contrast comparably well:
\[
R^2=
\vkFormPower,
\qquad
\vkFormLog,
\qquad
\vkFormLinear,
\]
respectively, while the no-rank model has $R^2=0$.

The data therefore establish a rank dependence but do not identify its functional form. The second-order branch of the untruncated concentration bound has exponent $1/2$, so the observed value lies between the two sufficient-bound branches.

\paragraph{Direction recovery versus downstream error.}
The scaling in \eqref{eq:law} concerns locating $u_\star$. When the same thresholds are instead defined through downstream excess error, the scaling largely disappears. In the weak-signal regime, Bayes error approaches chance, making excess error less sensitive to whether the discriminant direction has been accurately recovered.

\subsection{The threshold crossing depends on calibration}
\label{sec:notaconstant}

The ratio in Figure~\ref{fig:law} crosses one, but the location of that crossing is not a universal function of $(n,p)$.

Changing the angular success criterion moves the crossing by a factor of four across the criteria tested. Changing which covariance spike is strongest moves the crossing by two orders of magnitude. An anisotropic latent covariance changes the prefactor in a way that is not constant across ranks.

\begin{table}[ht]
\centering
\small
\setlength{\tabcolsep}{4pt}
\begin{tabular}{@{}lcc@{}}
\toprule
Condition & Crossing & Exponent \\
& ($n/p^2$) & \\
\midrule
$\const{10}^{\circ}$ criterion
&
$\vkCrossTen$
&
$\vkExpCritTen$
\\
$\const{15}^{\circ}$ criterion (baseline)
&
$\vkCrossFifteen$
&
$\vkExpCritFifteen$
\\
$\const{20}^{\circ}$ criterion
&
$\vkCrossTwenty$
&
$\vkExpCritTwenty$
\\
Strongest-spike discriminant
&
$\vkTaskACross$
&
---
\\
Anisotropic $\bSig$
&
not determined
&
---
\\
\bottomrule
\end{tabular}
\caption{Sensitivity of the threshold comparison. The crossing changes with the angular criterion and with the covariance-spike configuration. For the anisotropic row, the prefactor shift is not constant across the rank grid, so no single crossing is reported.}
\label{tab:crossing_shifts}
\end{table}

\subsection{Calibrating the angular criterion}
\label{sec:calibration}

A fixed angular threshold can confound recovery with the geometry of the searched subspace. In the anisotropic model, nuisance directions can carry much more projection variance than the discriminant. The empirical kurtosis landscape can then be flat over much of the subspace.

Calibrating to a common false-positive rate does not fully remove this effect because the null distribution changes its scale as well as its tail mass. Referring the criterion to the location of the null instead---using a fixed fraction of its median---removes the observed anomaly and leaves the measured exponent stable across the fractions tested (Appendix~\ref{app:sweeps}).

The calibration must also respect the different nulls of the two procedures. Removing the mean gap is a null for fourth-moment KPP, but it is not a null for PCA because the latent covariance still produces a population covariance spike:
\[
\operatorname{cov}(\yt)
=
s^2\Gt\Gt^\top+I.
\]
Calibrating the PCA arm against that null would therefore introduce an arbitrary gain threshold into the comparison.

We instead calibrate the fourth-moment side and evaluate the PCA side at a fixed leakage tolerance, reporting its small residual dependence on $\dsr$.

\section{Sample splitting does not repair the search}
\label{sec:splitting}

A separate validation sample can tell us whether a proposed direction is good. It cannot recover a direction that the search on the fitting sample never proposes.

We test this by appending observation coordinates that are independent of the latent regime. These coordinates leave the exact Bayes decisions unchanged but enlarge the unrestricted search space. Out-of-sample kurtosis for the unrestricted procedure remains near zero while the restricted procedure approaches $\kappa_\star$, and the unrestricted directions become nearly orthogonal to $u_\star$ as the ambient dimension increases.

\begin{figure}[t]
\centering
\includegraphics[width=\columnwidth]{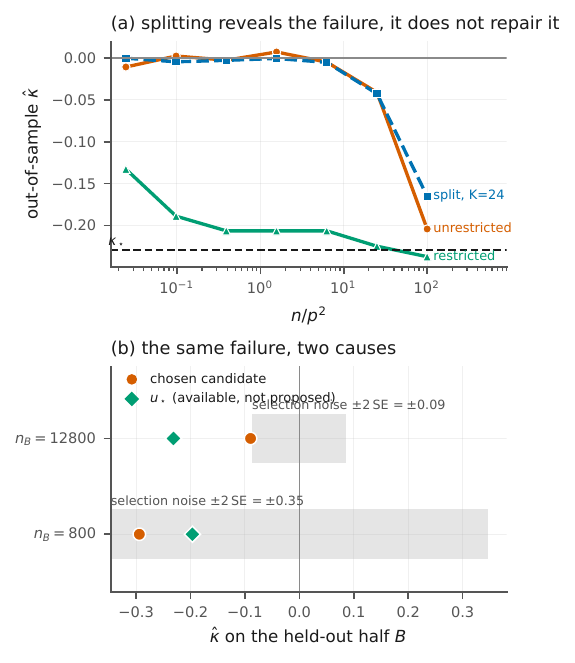}
\caption{Sample splitting exposes but does not remove the unrestricted-search failure. (a) Out-of-sample $\hat\kappa$ for the split search follows the unrestricted full-sample arm and remains near zero, while the restricted arm approaches $\kappa_\star$. (b) With $n_B=\const{800}$, validation noise can make a spurious candidate look better than $u_\star$; with $n_B=\const{12800}$, validation correctly ranks $u_\star$ above the selected candidate, but the search never proposed $u_\star$.}
\label{fig:splitting}
\end{figure}

There are therefore two distinct failures. With the smaller validation sample, selection noise can favor a spurious candidate. With the larger validation sample, selection correctly identifies $u_\star$ among the available candidates, but the fitting-stage search has already discarded it.

Doubling the number of points in the split arms leaves the unrestricted direction largely unchanged while improving the restricted arm. The unrestricted procedure is therefore search-limited in this regime, whereas the restricted procedure is sample-limited.

The same experiments compare exact and estimated subspaces. The operator-based subspace succeeds in almost every configuration. In the single failing cell, rank selection based on the Baik--Ben Arous--P\'ech\'e criterion recovers the discriminant, whereas supplying the exact rank spends additional dimensions on noise and performs worse. Thus the comparison changes when the discriminant is not the spike that determines subspace recovery; this is the scope condition discussed in Section ~\ref{sec:law}.

Additional operator-perturbation diagnostics, spurious-minimum-depth experiments, and optimizer-restart checks are reported in the supplement. They provide supporting evidence without changing the main scaling comparison.

\section{Related work}
\label{sec:related}

\paragraph{Projection pursuit.}
Projection pursuit originates with \citet{friedman1974projection} and \citet{huber1985projection}, and minimum-kurtosis indices for mixture separation with \citet{pena2001cluster}. Kurtosis is also a standard fourth-moment contrast in independent component analysis \citep{hyvarinen2000independent}.

\paragraph{High-dimensional projection pursuit.}
\citet{bickel2018projection} show that in high dimensions one can find directions in Gaussian data whose empirical projections resemble prescribed non-Gaussian laws. \citet{montanari2025overparametrizedlineardimensionalityreductions} characterize the Wasserstein radius of low-dimensional projections of i.i.d.\ Gaussian data in an overparameterized regime.

Our setting isolates a related structural question: when the observation model identifies a Gaussian complement known to contain no latent information, can that complement be removed without changing the population target? We give an exact kurtosis degeneracy, a population dilution identity, a finite-sample searched-subspace bound, and a comparison between oracle subspace restriction and data-driven subspace estimation.

\paragraph{Kurtosis-based discriminant estimation.}
\citet{radojicic2021large} derive large-sample properties for blind projection-pursuit estimators of the linear discriminant in two-group Gaussian mixtures, including kurtosis- and skewness-based estimators and their asymptotic covariance matrices. Their analysis is estimator-centric and asymptotic, whereas our focus is the finite-sample effect of the searched observation space.

\paragraph{Other projection indices.}
\citet{mukherjee2023wassersteinprojectionpursuitnongaussian} study Wasserstein projection pursuit and prove recovery under a linear $p/n$ scaling. Their index differs from the fourth-moment criterion studied here, but the same high-dimensional phenomenon---an empirical search over many directions---is relevant.

\paragraph{Subspace restriction.}
Likelihood-informed subspaces \citep{cui2014likelihood,zahm2022certified} reduce parameter-space dimension for specified inverse problems. Our restriction is different: it acts on observation-space directions in an unlabeled projection-pursuit search. \citet{francisci2026centralsubspacedatadepth} provide another example in which a depth criterion is optimized over a subspace rather than a single direction.

\section{Discussion and limitations}
\label{sec:limits}

The results separate three quantities that are easy to conflate: the intrinsic information available in the observations, the cost of estimating a useful search subspace, and the cost of finding a direction by minimizing a fourth-moment criterion.

A key consequence is that the information-theoretic problem can be easier than the fourth-moment search. In particular, the $\varsigma^{-8}$ dependence arises from the sufficient-bound analysis of the fourth-moment criterion, not from an information-theoretic lower bound for the full inference problem.

A two-point construction makes this distinction explicit. Index instances by the latent mean-gap axis and take two axes separated by $2\theta_0$. The Kullback--Leibler divergence is $\Theta(|\kappa_\star|)$ rather than $\Theta(\kappa_\star^2)$, because the hypotheses already differ at the second moment. The resulting two-point argument therefore gives $\varsigma^{-4}$, not $\varsigma^{-8}$. A likelihood-ratio test separates the same pair below the $\dsr/\kappa_\star^2$ scale. Thus this standard route does not yield a matching $\varsigma^{-8}$ lower bound for the full statistical problem.

A lower bound specific to minimizers of the empirical fourth-moment contrast remains open; such a result would have to control fluctuations of the contrast and their supremum over the searched set, where the search-dimension dependence also enters.

The scope is deliberately narrow. The theory and scaling comparison use two-component mixtures with the discriminant on a single left singular direction. Unequal mixture weights are covered by the population calculation, but $M>2$ is not. Extending the result to multiple components would require a different estimand and a corresponding subspace-recovery analysis.

The functional form of the search-dimension dependence and the exponent governing the depth of spurious minima also remain open. Finally, the theory concerns the exact global minimizer, while the experiments use multi-start gradient descent. The restart experiments in the supplement provide evidence that the reported thresholds are not optimizer artifacts, but they do not establish global optimization guarantees.

\emph{Recoverability and discoverability are different problems}.
When the observation model identifies a signal subspace, restricting an unlabeled projection-pursuit search to that subspace can remove a high-dimensional failure mode without changing the population target on the negative-kurtosis branch. The remaining cost is the separate problem of estimating that subspace or finding the direction within it.

\section*{AI Use Statement}

Generative AI tools were used during preparation of this manuscript for conceptual discussion and development, mathematical proof checking, literature-search assistance, drafting, editing, and organization. The authors reviewed and verified all AI-assisted material used in the manuscript, including mathematical arguments, citations, and interpretations. Generative AI was not used to generate experimental data or to determine the reported experimental results. The authors take full responsibility for the final content of the submission, including all mathematical statements, experimental results, citations, and interpretations. Generative AI systems are not authors of this work.


\bibliographystyle{plainnat}
\bibliography{references}


\onecolumn
\appendix
\aistatstitle{Subspace Restriction in Unlabeled Projection Pursuit:\\
A Scaling Law and Where It Fails --- Supplementary Material}

\section{Statements moved from the main text}
\label{app:restated}
We restate the three results summarized in the main text.

\begin{theorem}[Uniform concentration for the untruncated criterion]
\label{thm:untrunc}
Let $\mathcal S\subseteq\Rb^p$ have dimension $\dsr$ and let $U\subseteq\mathcal S\cap S^{p-1}$.
Put $L_n=\dsr\log(3n)+\log(18/\delta)$. Under the assumptions of Lemma~\ref{lem:floor} there are
absolute constants $C,C'$ such that, if $n\ge C'\sigma^8L_n^2$, then with probability at least
$1-\delta$,
\begin{equation}
\sup_{u\in U}|\hat\kappa(u)-\kappa(u)|
\le C\sigma^{10}\left[\sqrt{\tfrac{L_n}{n}}+\tfrac{L_n^2}{n}\right].
\label{eq:untrunc}
\end{equation}
Under the plug-in convention, $\hat\kappa_{\mathrm{impl}}+3=((n-1)/n)^2(\hat\kappa+3)$, and the resulting difference is controlled on the same event.
\end{theorem}

\begin{proposition}[Exact degeneracy of the unrestricted criterion]
\label{prop:degen}
Let $n$ be even and $p-r\ge n-1$. Under the whitened model, with probability one there is a unit
$b\in\Rsp^\perp$ such that
\begin{equation}
\hat\kappa(b)=-2,\qquad \kappa(b)=0,\qquad d(b,u_\star)=\pi/2
\label{eq:degen}
\end{equation}
whenever the population minimizer $u_\star$ is unique in $\Rsp$. If $r\le n-2$, then with
probability one no direction in $\Rsp$ attains the empirical floor.
\end{proposition}

\begin{proposition}[Spurious complement minima]
\label{prop:depth}
Let $f_1,\ldots,f_m$ be an orthonormal basis of $\Rsp^\perp$, with $m=p-r$, and put
$g_{ij}=f_j^\top\yt_i$. Under the assumptions below there are absolute constants $c,C,a>0$ such
that, for $n\ge C$, $m\ge C\log(1/\delta)$ and $\log(6m/\delta)\le\sqrt n$,
\begin{equation}
\inf_{b\in\Rsp^\perp\cap S^{p-1}}\hat\kappa(b)
\le -c\sqrt{\frac{\log(m\wedge n^a)}{n}}+C\frac{\log(m/\delta)}{n}.
\label{eq:depth}
\end{equation}
For $\hat\kappa_T$ the same bound holds on the no-clipping event described below; its failure
probability adds at most $2mn^{-3}$.
\end{proposition}

\begin{table}[ht]
\centering
\small
\begin{tabular}{rrrrrr}
\toprule
$p$ & $n$ & $n/p^2$ & $\varsigma_{\mathrm{KPP}}$ & gap-clean & budget \\
\midrule
1024 & 4096 & 0.004 & 0.924 & 0.905 & 0.875 \\
512 & 4096 & 0.016 & 0.868 & 0.868 & 0.868 \\
256 & 4096 & 0.062 & 0.911 & 0.911 & 0.911 \\
128 & 4096 & 0.250 & 1.098 & 1.098 & 1.098 \\
128 & 16384 & 1.000 & 0.819 & 0.786 & 0.773 \\
32 & 4096 & 4.000 & 0.908 & 0.908 & 0.906 \\
128 & 65536 & 4.000 & 0.621 & 0.621 & 0.621 \\
32 & 16384 & 16.000 & 0.711 & 0.711 & 0.712 \\
128 & 262144 & 16.000 & 0.490 & 0.490 & 0.492 \\
32 & 65536 & 64.000 & 0.610 & 0.610 & 0.612 \\
32 & 262144 & 256.000 & 0.493 & 0.493 & 0.493 \\
\bottomrule
\end{tabular}
\caption{The kurtosis threshold measured via the standard restart budget, on cells with an optimizer gap within the null tolerance, and re-searched at a \const{512}-fold larger budget. The implied crossings are within the bootstrap interval of one another.}
\label{tab:threeway}
\end{table}

\section{Proofs}
\label{app:proofs}
Throughout, $\Gt=\bS^{-1/2}\bH\bPhi$, $\Rsp=\mathrm{col}(\Gt)$, $P$ is the orthogonal projector onto
$\Rsp$, and $\yt=\Gt z+\tilde\varepsilon$ with $\tilde\varepsilon\sim\NPDF(0,I_p)$ independent of
$z$. The symbols $C,C',c$ denote absolute constants whose values change from line to line. We
assume $\sigma\ge1$ throughout ($\mathrm{Var}(u^\top\yt)\ge1$ by Lemma~\ref{lem:floor}).

\subsection{Lemma~\ref{lem:complement} (noise-only complement)}
Since $\Gt z\in\Rsp$ for every $z$, $(I-P)\Gt z=0$, and hence
\[
(I-P)\yt=(I-P)\tilde\varepsilon.
\]
Thus $(I-P)\yt$ is a function of $\tilde\varepsilon$ alone and is independent of $z$. It is Gaussian
with mean zero and covariance
\[
(I-P)I(I-P)^\top=I-P,
\]
making it standard Gaussian on $\Rsp^\perp$. Moreover,
\[
\mathrm{Cov}(P\tilde\varepsilon,(I-P)\tilde\varepsilon)=P(I-P)=0.
\]
Because the two vectors are jointly Gaussian, they are independent. Conditioning on $z$ and then
integrating establishes the independence of $P\yt$ and $(I-P)\yt$. \hfill$\square$

\subsection{Lemma~\ref{lem:dilution} (exact dilution)}
Let $u=\cos\theta\,a+\sin\theta\,b$ with $a\in\Rsp$, $b\in\Rsp^\perp$ unit vectors, and write
$X=a^\top\yt$, $Y=b^\top\yt$. By Lemma~\ref{lem:complement}, $Y\sim\NPDF(0,1)$ independently of
$X$, and
\[
u^\top\yt=\cos\theta\,X+\sin\theta\,Y.
\]
Let $v=\mathrm{Var}(X)$ and let $\mu_4(X)$ denote its fourth central moment, so
$\kappa_a=\mu_4(X)/v^2-3$. Since both variables are centred and independent,
\[
\mu_4(u^\top\yt)
=\cos^4\!\theta\,\mu_4(X)
+6\cos^2\!\theta\sin^2\!\theta\,v
+3\sin^4\!\theta ,
\]
and
\[
\mathrm{Var}(u^\top\yt)
=v\cos^2\theta+\sin^2\theta=:w.
\]
Therefore
\[
\kappa(u)+3
=
\frac{\cos^4\!\theta\big[(\kappa_a+3)v^{2}\big]
+6\cos^2\!\theta\sin^2\!\theta\,v
+3\sin^4\!\theta}{w^{2}}
=
\frac{\kappa_a v^{2}\cos^4\!\theta+3w^2}{w^{2}}
=
\kappa_a\frac{v^{2}\cos^4\!\theta}{w^{2}}+3 .
\]
Hence
\[
\kappa(u)=\kappa_a\lambda(\theta)^2,
\qquad
\lambda(\theta)=\frac{v\cos^2\!\theta}{w}\in[0,1],
\]
with $\lambda=1$ if and only if $\sin\theta=0$, and $\lambda=0$ if and only if $\cos\theta=0$.
\hfill$\square$

\paragraph{Sign dependence (Lemma~\ref{lem:dilution}).}
The identity $\kappa(u)=\kappa_a\lambda^2$ preserves the sign of $\kappa_a$ while reducing its
magnitude. If $\kappa_a<0$, the minimum over $\theta$ is attained at
$\lambda=1$. Every population minimizer in the corresponding plane lies in $\Rsp$, and restricting
the search does not change the optimum. If $\kappa_a>0$, the minimum is attained at $\lambda=0$;
the criterion is driven into $\Rsp^\perp$, where $\kappa\equiv0$ and the projection contains no
regime information.

\subsection{Lemma~\ref{lem:floor} (variance floor and uniform sub-Gaussianity)}
\emph{Variance floor.}
The variables $z$ and $\tilde\varepsilon$ are independent, so for every unit $u$,
\[
\mathrm{Var}(u^\top\yt)
=
\mathrm{Var}(u^\top\Gt z)
+\mathrm{Var}(u^\top\tilde\varepsilon)
=
\mathrm{Var}(u^\top\Gt z)+\|u\|^2
\ge1.
\]

\emph{Tails.}
Write $\bar\mu=\sum_k\pi_k\mu_k$ and
\[
X_u=u^\top(\yt-\Eb\yt)=W_u+u^\top\tilde\varepsilon,
\qquad
W_u=u^\top\Gt(z-\bar\mu).
\]
Conditionally on component $k$, $W_u\sim\NPDF(a_k,b_k)$ with
\[
a_k=u^\top\Gt(\mu_k-\bar\mu),
\qquad
b_k=u^\top\Gt\bSig_k\Gt^\top u.
\]
For every unit $u$,
\[
|a_k|\le A:=\|\Gt\|_{\mathrm{op}}\max_k\|\mu_k-\bar\mu\|,
\qquad
b_k\le B^2:=\|\Gt\|_{\mathrm{op}}^2\max_k\|\bSig_k\|_{\mathrm{op}}.
\]
Therefore, for $t>A$,
\[
\Pr(|W_u|\ge t)
\le
\sum_k\pi_k\Pr\big(|\NPDF(a_k,b_k)|\ge t\big)
\le
2\exp\!\Big(-\frac{(t-A)^2}{2B^2}\Big),
\]
which gives $\|W_u\|_{\psi_2}\le c(A+B)$. Adding the independent $\NPDF(0,1)$ term and using
subadditivity of the $\psi_2$ norm yields
\[
\|X_u\|_{\psi_2}\le c(1+A+B),
\]
equivalently the stated bound with
\[
\sigma^2 =
c\left(
1+\|\Gt\|_{\mathrm{op}}^2
\left[
\max_k\|\bSig_k\|_{\mathrm{op}}
+\max_k\|\mu_k-\bar\mu\|^2
\right]\right).
\]
Consequently, $\Eb X_u^{2q}\le C_q\sigma^{2q}$ for every fixed $q$. This bound depends only on $\|\Gt\|_{\mathrm{op}}$, $\max_k\|\bSig_k\|_{\mathrm{op}}$, and $\max_k\|\mu_k-\bar\mu\|$, none of which depend on $p$. \hfill$\square$

In the laboratory setting of Section~\ref{sec:splitting},
$\|\Gt\|_{\mathrm{op}}=1$, $\bSig_k=s^2I_r$, and
$\|\mu_k-\bar\mu\|=\Delta/2$. Hence
\[
\sigma^2=c(1+s^2+\Delta^2/4),
\]
which is exactly invariant under the $p$ and $r$ sweep.

\subsection{Theorem~\ref{thm:conc} (uniform concentration)}
Fix a linear subspace $\mathcal{S}$ with $\dim\mathcal{S}=\dsr$, let $U\subseteq\mathcal{S}\cap S^{p-1}$, write
$\Pi_{\mathcal{S}}$ for the orthogonal projector onto $\mathcal{S}$, and set
$X_i=\yt_i-\Eb\yt$, $z_i(u)=u^\top X_i$. Let $T=\sigma\sqrt{8\log n}$ and
\[
L=L_n(\delta)=\dsr\log(3n)+\log(12/\delta).
\]
We bound $\sup_{u\in U}|\hat m_{q,T}(u)-\mu_q(u)|$ for $q\in\{2,4\}$, where $\mu_q(u)=\Eb X_u^q$.

\paragraph{Step 1: truncation bias.}
By Cauchy--Schwarz and Lemma~\ref{lem:floor}, for $q\in\{2,4\}$,
\[
\big|\Eb\varphi_T(X_u)^q-\mu_q(u)\big|
\le
\Eb\big[|X_u|^q\mathbf{1}\{|X_u|>T\}\big]
\le
\big(\Eb X_u^{2q}\big)^{1/2}\Pr(|X_u|>T)^{1/2}
\le
C\sigma^{q}\,n^{-2},
\]
using $\Pr(|X_u|>T)\le 2e^{-T^2/2\sigma^2}=2n^{-4}$.
For $q=2$,
\[
\Eb\varphi_T(X_u)^2
\ge
\mu_2(u)-C\sigma^2n^{-2}
\ge
1-C\sigma^2n^{-2}.
\]

\paragraph{Step 2: empirical centring.}
Since $U\subseteq S$,
\[
\sup_{u\in U}|u^\top(\bar y_n-\Eb\yt)|
\le
\|\Pi_{\mathcal{S}}(\bar y_n-\Eb\yt)\|
=:\Delta_n.
\]
The vectors $\Pi_{\mathcal{S}}X_i$ are i.i.d., mean zero, and have
$\sigma$-sub-Gaussian one-dimensional marginals. A $1/2$-net of
$\mathcal{S}\cap S^{p-1}$ with cardinality at most $5^{\dsr}$, together with Hoeffding's inequality,
gives, with probability at least $1-\delta/6$,
\[
\Delta_n
\le
C\sigma\sqrt{\frac{\dsr+\log(6/\delta)}{n}}
\le
C\sigma\sqrt{\frac{L}{n}}.
\]
Because $\varphi_T$ is $1$-Lipschitz and
$|a^q-b^q|\le qT^{q-1}|a-b|$ for $a,b\in[-T,T]$,
\[
\Big|
\hat m_{q,T}(u)
-\tfrac1n\textstyle\sum_i\varphi_T(z_i(u))^q
\Big|
\le
4T^{3}\Delta_n
\le
C\sigma^{4}(\log n)^{3/2}\sqrt{L/n}
\]
uniformly over $u\in U$, using $T^3=\sigma^3(8\log n)^{3/2}$ and $\sigma\ge1$.

\paragraph{Step 3: net and Bernstein.}
Let $N$ be a $\gamma$-net of $U$ in the Euclidean metric with $\gamma=1/n$. Since $U$ lies in the
unit sphere of a $\dsr$-dimensional subspace, $|N|\le(3n)^{\dsr}$.
Fix $u\in N$ and $q\in\{2,4\}$. The summands
$\xi_i=\varphi_T(z_i(u))^q$ are i.i.d.\ in $[0,T^q]$ with
$\mathrm{Var}(\xi_i)\le\Eb X_u^{2q}\le C\sigma^{2q}$.
Bernstein's inequality gives, with probability at least $1-\delta'$,
\[
\Big|
\tfrac1n\textstyle\sum_i\xi_i-\Eb\xi
\Big|
\le
\sqrt{\frac{2C\sigma^{2q}\log(2/\delta')}{n}}
+
\frac{2T^{q}\log(2/\delta')}{3n}.
\]
Taking $\delta'=\delta/(6|N|)$ gives $\log(2/\delta')\le L$. A union bound over $N$ and over
$q\in\{2,4\}$ yields
\[
\max_{u\in N}
\Big|
\tfrac1n\textstyle\sum_i\varphi_T(z_i(u))^q
-\Eb\varphi_T(X_u)^q
\Big|
\le
C\sigma^{q}\sqrt{\frac{L}{n}}
+
C\,T^{q}\frac{L}{n}
\le
C\sigma^{4}
\Big[
\sqrt{\frac{L}{n}}
+
(\log n)^{2}\frac{L}{n}
\Big]
\]
with probability at least $1-\delta/3$.
For the discretization, if $\|u-u'\|\le\gamma$,
\[
|\varphi_T(z_i(u))^q-\varphi_T(z_i(u'))^q|
\le
qT^{q-1}\gamma\|\Pi_{\mathcal{S}}X_i\|.
\]
Consequently,
\[
\sup_{\|u-u'\|\le\gamma}
\Big|
\tfrac1n\textstyle\sum_i\varphi_T(z_i(u))^q
-
\tfrac1n\sum_i\varphi_T(z_i(u'))^q
\Big|
\le
\frac{4T^{3}}{n}\cdot
\frac1n\sum_i\|\Pi_{\mathcal{S}}X_i\|.
\]
The average is bounded by Bernstein's inequality:
\[
\frac1n\sum_i\|\Pi_{\mathcal{S}}X_i\|^2
\le
C\sigma^2(\dsr+\log(6/\delta))
\le
C\sigma^2L
\]
with probability $1-\delta/6$. The discretization contribution is $O\big(\sigma^4(\log n)^{3/2}\sqrt{L}/n\big)$, which is dominated by the preceding terms.

\paragraph{Step 4: assembling the moments.}
Combining Steps 1--3 gives, with probability at least $1-\delta/2$, simultaneously for
$q\in\{2,4\}$,
\begin{equation}
\sup_{u\in U}\big|\hat m_{q,T}(u)-\mu_q(u)\big|
\le
E_n
:=
C\sigma^{4}\Big[
(\log n)^{3/2}\sqrt{\tfrac{L}{n}}
+
(\log n)^{2}\tfrac{L}{n}
\Big].
\label{eq:momentbound}
\end{equation}

\paragraph{Step 5: the ratio.}
The condition
$n\ge C'\sigma^{8}(\log n)^3L$ makes $E_n\le1/8$. This gives
$\hat m_{2,T}(u)\ge1/2$ for all $u\in U$.
Writing $a=\hat m_{4,T}$, $b=\hat m_{2,T}$, $a_0=\mu_4$, and $b_0=\mu_2$,
\[
\Big|
\frac{a}{b^{2}}-\frac{a_0}{b_0^{2}}
\Big|
\le
\frac{|a-a_0|}{b^{2}}
+
a_0\frac{|b_0^{2}-b^{2}|}{b^{2}b_0^{2}}
\le
4|a-a_0|
+
C\sigma^{4}\cdot4\,(b+b_0)\,|b-b_0|
\le
C\sigma^{6}E_n,
\]
which yields the stated bound. \hfill$\square$

\paragraph{Remarks on the proof.}
(i) Truncation bounds the summands for Bernstein's inequality. Removing truncation (Theorem~\ref{thm:untrunc}) introduces a sub-Weibull tail, yielding a worse $L^2/n$ second-order term (quantified in Remark~\ref{rem:trunc}) while preserving the leading $\sqrt{\dsr/n}$ behavior.
(ii) The variance floor in Lemma~\ref{lem:floor} ensures Step~5 is uniform over the whole sphere.
(iii) The powers of $\sigma$ and logarithmic factors are not optimized and could be improved via sharper chaining.
(iv) The bound applies to the exact global minimizer of $\hat\kappa_T$. Lemma~\ref{lem:noclip} provides the condition under which this coincides with the untruncated criterion optimized numerically, evaluated in Appendix~\ref{app:fs}.

\subsection{Lemma~\ref{lem:noclip} (uniform non-clipping)}
\begin{lemma}[Uniform non-clipping]
\label{lem:noclip}
Let $\mathcal{S}\subseteq\Rb^p$ be a subspace with orthogonal projector $\Pi_{\mathcal{S}}$, and put
$z_i=\yt_i-\bar y_n$. If
\begin{equation}
\max_{i\le n}\big\|\Pi_{\mathcal{S}}z_i\big\|\;\le\;T,
\label{eq:trunccond}
\end{equation}
then $\hat\kappa_T(u)=\hat\kappa(u)$ for \emph{every} unit $u\in\mathcal{S}$.
\end{lemma}
\emph{Proof.}
For every unit $u\in\mathcal{S}$, $|u^\top z_i| = |u^\top\Pi_{\mathcal{S}}z_i| \le \|\Pi_{\mathcal{S}}z_i\| \le T$. Thus $\varphi_T$ is the identity on all $n$ projections, and $\hat\kappa_T(u)=\hat\kappa(u)$. \hfill$\square$

\subsection{Theorem~\ref{thm:untrunc} (the untruncated criterion)}
Write $X_i=\yt_i-\Eb\yt$, $X_{u,i}=u^\top X_i$, and $\mu_q(u)=\Eb X_u^q$.
Let $\tilde m_q(u)=\frac1n\sum_iX_{u,i}^{q}$ and $\hat m_q(u)=\frac1n\sum_i(u^\top(\yt_i-\bar y_n))^q$.
Put $L=L_n=\dsr\log(3n)+\log(18/\delta)$.

\paragraph{Step 0: tail class of $X_u^q$.}
For a unit $u\in\mathcal{S}$, Lemma~\ref{lem:floor} gives $\|X_u\|_{\psi_2}\le C\sigma$.
Using the identity $\|Z^{q}\|_{\psi_{\alpha/q}}=\|Z\|_{\psi_\alpha}^{q}$, $X_u^q$ is sub-Weibull with shape $\alpha=2/q$ and $\|X_u^q\|_{\psi_{2/q}}\le C\sigma^q$.

\paragraph{Step 1: pointwise concentration.}
Applying Theorem~3.1 of \citet{kuchibhotla2022moving} to $\xi_i=X_{u,i}^{q}-\mu_q(u)$ yields, with probability at least $1-2e^{-t}$,
\begin{equation}
\big|\tilde m_q(u)-\mu_q(u)\big|
\le
C\sigma^{q}\Big[
\sqrt{\tfrac{t}{n}}
+
\tfrac{t^{q/2}}{n}
\Big].
\label{eq:subweibull}
\end{equation}
For $t\ge1$ and $\sigma\ge1$, all three values of $q\in\{2,3,4\}$ are bounded by $C\sigma^{4}\big[\sqrt{t/n}+t^{2}/n\big]$.

\paragraph{Step 2: net.}
Let $N$ be a $\gamma$-net of $U$ with $\gamma=1/n$. Since $|N|\le(3n)^{\dsr}$, applying \eqref{eq:subweibull} at each $u\in N$ and each $q\in\{2,3,4\}$ yields, with probability at least $1-\delta/3$,
\begin{equation}
\max_{u\in N}\max_{2\le q\le4}
\big|\tilde m_q(u)-\mu_q(u)\big|
\le
C\sigma^{4}
\Big[
\sqrt{\tfrac{L}{n}}
+
\tfrac{L^{2}}{n}
\Big]
=:
E_n .
\label{eq:rawdev}
\end{equation}

\paragraph{Step 3: extension from the net to $U$.}
The mean-value theorem applied to $x\mapsto x^q$ gives
\begin{equation}
\big|X_{u,i}^{q}-X_{u',i}^{q}\big|
\le
q\,\big|X_{u,i}-X_{u',i}\big|\,
\max\big(|X_{u,i}|,|X_{u',i}|\big)^{q-1}.
\label{eq:mvt}
\end{equation}
A $1/2$-net of $\mathcal{S}\cap S^{p-1}$ gives, with probability at least $1-\delta/6$,
\begin{equation}
\max_{i\le n}\|\Pi_{\mathcal{S}}X_i\|
\le
C\sigma\Big(\sqrt{\dsr}+\sqrt{\log(6n/\delta)}\Big)
\le
C\sigma\sqrt{L}.
\label{eq:maxnorm}
\end{equation}
The resulting discretization error is $4\gamma\max_i\|\Pi_{\mathcal{S}}X_i\|^{4} \le C\sigma^{4}L^{2}/n$, which extends \eqref{eq:rawdev} to all of $U$.

\paragraph{Step 4: sample centring.}
Let $\Delta(u)=u^\top(\bar y_n-\Eb\yt)=\tilde m_1(u)$.
With probability at least $1-\delta/6$,
\begin{equation}
\sup_{u\in U}|\Delta(u)|
\le
C\sigma\sqrt{L/n}.
\label{eq:deltabound}
\end{equation}
Exact binomial identities give $\hat m_2=\tilde m_2-\Delta^{2}$ and $\hat m_4=\tilde m_4-4\Delta\tilde m_3+6\Delta^{2}\tilde m_2-3\Delta^{4}$.
Substituting the moment bounds yields
\[
\sup_{u\in U}|\hat m_q(u)-\mu_q(u)|
\le
CE_n,
\qquad q\in\{2,4\}.
\]

\paragraph{Step 5: the ratio.}
The condition $n\ge C'\sigma^{8}L^{2}$ makes $CE_n\le1/2$.
With $a=\hat m_4$ and $b=\hat m_2$, $b\ge 1/2$ and $b\le C\sigma^{2}$ uniformly. Consequently,
\[
\Big|
\frac{a}{b^{2}}-\frac{a_0}{b_0^{2}}
\Big|
\le
4|a-a_0|
+
C\sigma^{4}\cdot C\sigma^{2}\cdot4|b-b_0|
\le
C\sigma^{6}E_n,
\]
which yields \eqref{eq:untrunc}. The statement for $\hat\kappa_{\mathrm{impl}}$ follows from $\hat\kappa_{\mathrm{impl}}+3 = ((n-1)/n)^2(\hat\kappa+3)$. \hfill$\square$

\paragraph{Second-order term and scope.}
The leading term $\sqrt{L/n}$ avoids the $(\log n)^{3/2}$ truncation factor of Theorem~\ref{thm:conc}. The second-order term $L^{2}/n$ cannot generally be replaced by one linear in $L$: for $\mathcal{S}=\Rb^{p}$, a single observation yields $\sup_{S^{p-1}}|\hat\kappa-\kappa|\ge c\,m^{2}/n$ (Remark~\ref{rem:trunc}). The truncated criterion removes this witness, permitting the linear-in-$L$ second term in Theorem~\ref{thm:conc}.

For $\mathcal{S}=\Rsp$, Theorem~\ref{thm:untrunc} covers the implemented criterion without a clipping certificate. For $\mathcal{S}=\Rb^{p}$, the sufficient sample size for an angular guarantee becomes quadratic in $\dsr$. Substituting \eqref{eq:untrunc} into the argument of Corollary~\ref{cor:nstar} requires:
\begin{equation}
n\;\ge\;\frac{C\sigma^{20}V^{2}L_n}{\kappa_\star^{2}\sin^{4}\theta_0}
\quad\text{and}\quad
n\;\ge\;\frac{C\sigma^{10}VL_n^{2}}{|\kappa_\star|\sin^{2}\theta_0},
\label{eq:nstaruntrunc}
\end{equation}
whose ratio is $2C\sigma^{10}V / (|\kappa_\star|\sin^{2}\theta_0L_n)$. For the untruncated criterion, the ratio exceeds $1$ only while $L_n\lesssim\sigma^{10}V$. Removing truncation therefore costs a factor $\dsr$ in the sufficient sample size.

\subsection{Lemma~\ref{lem:growth} (global quadratic growth)}
Let $\Gt$ have orthonormal columns, $\bSig_1=\bSig_2=s^2I_r$, $\mu_1-\mu_2=\Delta e_1$, $\pi_1=\pi_2=1/2$, and put $\beta=s^2+\Delta^2/4$ and $V=1+\beta$.
For a unit $u$, write $g=\Gt^\top u\in\Rb^r$. The projection $u^\top\yt$ is a balanced two-component Gaussian mixture with mean gap $\Delta g_1$ and total variance
\begin{equation}
V(u)=1+s^2\|g\|^2+\Delta^2g_1^2/4 .
\label{eq:Vu}
\end{equation}
Setting $p=q=1/2$ and $v_1=v_2$ in Eq.~\eqref{eq:kappa} gives
\begin{equation}
\kappa(u)
=
-\frac{\Delta^{4}g_1^{4}}{8\,V(u)^{2}} .
\label{eq:kappalab}
\end{equation}
Write $t=g_1^2$ and $w=\|g\|^2$. Then
\[
|\kappa|
=
\frac{\Delta^4t^2}
{8(1+s^2w+\Delta^2t/4)^2}
\]
is strictly increasing in $t$ and strictly decreasing in $w$. It is maximized at $t=w=1$, yielding $u_\star=\pm\Gt e_1$ and $\kappa_\star=-\Delta^4/(8V^2)$. Because $\Gt^\top\Gt=I_r$, $t=\cos^2 d(u,u_\star)$.

Using $w\ge t$ in the denominator of \eqref{eq:kappalab},
\[
\kappa(u)
\ge
-\frac{\Delta^{4}t^{2}}{8(1+\beta t)^{2}}
=
-\frac{\Delta^4}{8}h(t)^2,
\qquad
h(t):=\frac{t}{1+\beta t}.
\]
Hence
\[
\kappa(u)-\kappa(u_\star)
\ge
\frac{\Delta^4}{8}\cdot
\frac{(1-t)\,(1+t+2\beta t)}
{(1+\beta)^{2}(1+\beta t)^{2}} .
\]
The function $\phi(t)=\frac{1+t+2\beta t}{(1+\beta t)^{2}}$ attains its minimum on $[0,1]$ at an endpoint: $\min\{\phi(0),\phi(1)\} = \min\{1,\,2/(1+\beta)\}$. Since $1-t=\sin^2 d(u,u_\star)$,
\[
\kappa(u)-\kappa(u_\star)
\ge
|\kappa_\star|
\min\!\big(1,2/V\big)
\sin^{2}d(u,u_\star),
\]
establishing Eq.~\eqref{eq:growth}. As $d\to0$, $\phi(t)\to2/(1+\beta)$, matching the local curvature $2|\kappa_\star|/V$ obtained by expanding Lemma~\ref{lem:dilution} about $\theta=0$. \hfill$\square$

\subsection{Corollary~\ref{cor:nstar}}
Let $\eta=\sup_{u\in U}|\hat\kappa_T(u)-\kappa(u)|$, let $\hat u\in\arg\min_U\hat\kappa_T$, and let $u_\star\in\arg\min_U\kappa$ with $u_\star\in U$.
Because $u_\star$ is a population object and $\hat\kappa_T$ is a sample statistic, a two-sided deviation bound is required:
\[
\kappa(\hat u)-\kappa(u_\star)
=
[\kappa(\hat u)-\hat\kappa_T(\hat u)]
+
[\hat\kappa_T(\hat u)-\hat\kappa_T(u_\star)]
+
[\hat\kappa_T(u_\star)-\kappa(u_\star)]
\le
2\eta .
\]
Combining this with Lemma~\ref{lem:growth} gives
\[
\sin^2 d(\hat u,u_\star)
\le
\frac{2\eta}
{|\kappa_\star|\min(1,2/V)}.
\]
For $V\ge2$, $d(\hat u,u_\star)\le\theta_0$ follows whenever $\eta\le\tau:=|\kappa_\star|\sin^{2}\theta_0/V$. Theorem~\ref{thm:conc} bounds $\eta$ by two terms. Bounding each by $\tau/2$ gives:
\begin{align}
\text{(i)}\quad
C\sigma^{10}(\log n)^{3/2}\sqrt{L_n/n}\le\tfrac{\tau}{2}
&\iff
n\ \ge\
\frac{4C^{2}\sigma^{20}(\log n)^{3}L_n}{\tau^{2}},
\label{eq:corT1}\\[2pt]
\text{(ii)}\quad
C\sigma^{10}(\log n)^{2}L_n/n\le\tfrac{\tau}{2}
&\iff
n\ \ge\
\frac{2C\sigma^{10}(\log n)^{2}L_n}{\tau}.
\label{eq:corT2}
\end{align}
Condition \eqref{eq:corT1} establishes Eq.~\eqref{eq:nstar}. Condition \eqref{eq:corT1} implies \eqref{eq:corT2} whenever $(2C\,\sigma^{10}\,V\log n) / (|\kappa_\star|\sin^{2}\theta_0) \ge1$, which holds under the stated conditions ($\sigma\ge1, V\ge2, |\kappa_\star|\le2, n\ge3$). Thus Eq.~\eqref{eq:nstar} controls both terms. \hfill$\square$

\subsection{Proposition~\ref{prop:degen} (exact degeneracy)}
\paragraph{The empirical floor.}
For any sample and unit $u$, write $x_i=u^\top(\yt_i-\bar y_n)$. Cauchy--Schwarz gives $\frac1n\sum x_i^4 \ge (\frac1n\sum x_i^2)^2$, so $\hat\kappa(u)\ge-2$. Equality requires the $x_i$ to take the two values $\pm c$ in equal numbers, requiring $n$ to be even.

\paragraph{Attainment in the complement.}
Let $w_i=(I-P)\yt_i$. By Lemma~\ref{lem:complement}, these are i.i.d.\ $\NPDF(0,I_m)$ on $\Rsp^\perp$, where $m=p-r$. Let $W_c$ be the $n\times m$ matrix of centred $w_i$. When $m\ge n-1$, its image is exactly the sum-zero hyperplane $\mathcal{H}$.

Choose a balanced $\varepsilon\in\{\pm1\}^n$. Since $\varepsilon\in\mathcal{H}$, there exists $b_0\in\Rsp^\perp$ satisfying $W_cb_0=\varepsilon$. Set $b=b_0/\|b_0\|$. The centred projections along $b$ are $\varepsilon/\|b_0\|$, which are two-valued and symmetric, yielding $\hat\kappa(b)=-2$. The projections satisfy:
\begin{equation}
1/\|b_0\|
\le
\|W_c\|_{\mathrm{op}}/\sqrt n .
\label{eq:witnessproj}
\end{equation}
For the truncated criterion, the conclusion requires $\|W_c\|_{\mathrm{op}}\le T\sqrt n$. The Gaussian operator-norm bound ensures this holds with probability at least $1-\delta$ when $m\le c\,\sigma^{2}n\log n$. The untruncated degeneracy applies unconditionally.

Because $b\in\Rsp^\perp$, Lemma~\ref{lem:complement} yields $\kappa(b)=0$ and $\langle b,u_\star\rangle=0$. The vector $(u_\star^\top(\yt_i-\bar y_n))_i$ has a continuous density on $\mathcal{H}$, making $\hat\kappa(u_\star)>-2$ almost surely. For odd $n$, the balanced vector with one zero entry yields $\hat\kappa(b)=-2+1/(n-1)$.

\paragraph{Non-degeneracy of the restricted criterion.}
The centred projections attainable from $\Rsp$ form the image of an $r$-dimensional space. For $r\le n-2$, this $r$-dimensional random subspace lies in general position. The intersection with the rays defined by the sign patterns $\varepsilon$ has probability zero. \hfill$\square$

\subsection{Proposition~\ref{prop:depth} (spurious depth)}
Write $g_{ij}=f_j^\top\yt_i$. By Lemma~\ref{lem:complement}, $g_{ij}$ are i.i.d.\ $\NPDF(0,1)$ over both indices. Put $L=\log(6m/\delta)$.

\paragraph{Step 0: truncation.}
Since $g_{ij}-\bar g_j$ is Gaussian with variance $1-1/n\le1$,
$\Pr(|g_{ij}-\bar g_j|>T)\le 2n^{-4}$ for $\sigma\ge1$. A union bound over the $nm$ centred projections yields $\hat\kappa_T(f_j)=\hat\kappa(f_j)$ simultaneously for all $j$, with failure probability at most $2mn^{-3}$.

\paragraph{Step 1: raw moments.}
Define $\tilde m_q(f_j)=\frac1n\sum_i g_{ij}^{\,q}$ and $\tilde\kappa_j = \tilde m_4(f_j)/\tilde m_2(f_j)^2-3$. Each $\tilde m_q(f_j)$ is an average of $n$ i.i.d.\ terms, and the columns $j$ are independent.

\paragraph{Step 2: empirical centring.}
Expanding $(g_{ij}-\bar g_j)^q$ provides the identities for sample-centred moments $\hat m_2$ and $\hat m_4$. A union bound gives $\max_j|\bar g_j|\le \beta:=\sqrt{2L/n}$ with probability $1-\delta/6$. Applying Theorem~3.1 of \citet{kuchibhotla2022moving} yields:
\begin{equation}
\max_j|\tilde m_2-1|
\le
C\sqrt{L/n}+CL/n,\qquad
\max_j|\tilde m_3|
\le
C\sqrt{L/n}+CL^{3/2}/n.
\label{eq:rawmoments}
\end{equation}
Assuming $L=\log(6m/\delta)\le\sqrt n$, the second terms in \eqref{eq:rawmoments} are dominated. This gives $|\hat m_2-\tilde m_2|\le\beta^{2}$ and $|\hat m_4-\tilde m_4|\le C\beta^{2}$. Since $x\mapsto a/x^2$ is Lipschitz,
\begin{equation}
\max_{j\le m}\big|\hat\kappa(f_j)-\tilde\kappa_j\big|
\le
\frac{2CL}{n}.
\label{eq:centcost}
\end{equation}

\paragraph{Step 3: Hermite expansion of the raw statistic.}
Putting $\tilde a_j=\tilde m_2(f_j)-1$ and $\tilde S_j=\frac1n\sum_i(g_{ij}^4-6g_{ij}^2+3)$ gives $\tilde m_4=3+6\tilde a_j+\tilde S_j$. Therefore $\tilde\kappa_j = \tilde S_j-3\tilde a_j^{2}-2\tilde a_j\tilde S_j+O(\tilde a_j^{3})$. By \eqref{eq:rawmoments}, $\max_j|\tilde\kappa_j-\tilde S_j|\le CL/n$.

\paragraph{Step 4: lower tail.}
The variable $H_4(g)$ has mean $0$, variance $24$, and is Cram\'er-regular. Berry--Esseen bounds for $\tilde S_j$ give $\Pr(\tilde S_j\le-t\sqrt{24/n}) \ge \Phi(-t)-C_{\mathrm{BE}}n^{-1/2}$. Independence of the columns gives $\Pr\big(\min_j\tilde S_j>-t\sqrt{24/n}\big) \le \exp(-m\Phi(-t)/2)$, which is at most $\delta/6$ for $t\le \sqrt{2\log\!\big(m/(2\log(6/\delta))\big)}$.

\paragraph{Step 5: combination.}
On the intersection of the events,
\[
\inf_{b\in\Rsp^{\perp}\cap S^{p-1}}\hat\kappa_T(b)
\le
\min_{j\le m}\tilde S_j + O(L/n).
\]
The lower-tail bound $\min_j\tilde S_j \le -t\sqrt{24/n}$ with $t \asymp\sqrt{\log(m\wedge n^a)}$ establishes Eq.~\eqref{eq:depth}. \hfill$\square$

\subsection{Remark~\ref{rem:trunc} (necessity of truncation)}
\begin{remark}[The truncation level is ambient-dependent]
\label{rem:trunc}
Let $z_i=(I-P)(\yt_i-\bar y_n)$ and let $i_0$ maximise $\|z_i\|$. Taking $b$ along $z_{i_0}$ gives $\hat\kappa(b)+3\ge\|z_{i_0}\|^4/(n\lambda_{\max}^2)$, yielding $\sup_{S^{p-1}}|\hat\kappa-\kappa|\ge c\,m^2/n-3$ with high probability for $m\le n$. A single observation breaks the untruncated criterion as $p$ grows (evaluated in Appendix~\ref{app:lb}).
\end{remark}

\subsection{Scope and limitations of the theoretical results}
\label{app:population}
Propositions~\ref{prop:degen} and~\ref{prop:depth} identify an ambient-dimension-dependent failure mode specific to minimum-kurtosis projection pursuit over $S^{p-1}$. Lemma~\ref{lem:growth} and Corollary~\ref{cor:nstar} assume balanced weights and equal spherical component covariances, which Lemmas~\ref{lem:complement}--\ref{lem:floor} and Theorem~\ref{thm:conc} do not require. Without global growth, objective concentration holds but yields only the local expansion of Lemma~\ref{lem:dilution}.

The theory bounds a sufficient regime and a failure regime. Corollary~\ref{cor:nstar} establishes sufficiency for $n\gtrsim\dsr\,\mathrm{polylog}$. Proposition~\ref{prop:degen} establishes failure for even $n\le p-r+1$, and Proposition~\ref{prop:depth} establishes a spurious depth of $\asymp\sqrt{\log(m\wedge n^{a})/n}$ for $\hat\kappa_T$ under $m\le n^2$. The regime beyond $n=p$ is characterized by the measured depth law and certificates. The theoretical exponent in \eqref{eq:nstar} is linear in $\dsr$ (up to logarithmic factors), which is compatible with the steeper measured discovery thresholds over the finite tested range due to the unevaluated absolute constant.

Lemmas~\ref{lem:complement}--\ref{lem:floor}, Theorem~\ref{thm:conc}, and Propositions~\ref{prop:degen}--\ref{prop:depth} hold for general $M$-component mixtures. Proposition~\ref{prop:degen} requires a unique population minimizer inside $\Rsp$. The explicit minimizer, Lemma~\ref{lem:growth}, and Corollary~\ref{cor:nstar} are stated for $M=2$, and for $M>2$ the minimum-kurtosis direction need not be unique. All measurements reported in this paper use $M=2$.

\section{Numerical checks}
\label{app:fs}
\paragraph{(A) Lemma~\ref{lem:growth}.}
The inequality \eqref{eq:growth} is evaluated exactly in closed form over \vFSGrowthConfigs{} configurations spanning $p\in\{8,16,32,64\}$, $r\in\{2,4,8,16\}$, $\Delta\in\{0.8,1.6,3.0\}$, and $s\in\{0.25,0.5,1.0\}$. The inequality holds throughout all configurations. The smallest slack is \vFSGrowthSlack, occurring at $d=0$ where it is an equality.

\paragraph{(B) Theorem~\ref{thm:conc}.}
The uniform deviation $\eta(\dsr,n)$ is estimated using $4000$ random directions followed by gradient ascent on $|\hat\kappa_T-\kappa|$ from the $24$ best directions, separately over $S^{p-1}$ and $\Rsp\cap S^{p-1}$. The configurations use $r=4$, $p\in\{8,16,32,64\}$, and $n \in [400, 25600]$ over 8 seeds. The fitted decay in $n$ is $\vFSSlopeFull$ $[\vFSSlopeFullLo,\vFSSlopeFullHi]$ over the full sphere and $\vFSSlopeRow$ $[\vFSSlopeRowLo,\vFSSlopeRowHi]$ over the signal subspace, matching the $-1/2$ exponent of the leading term in Eq.~\eqref{eq:conc}.

The dependence on ambient dimension is $\vFSSlopeFullInP$ $[\vFSSlopeFullInPLo,\vFSSlopeFullInPHi]$ for the full sphere and $\vFSSlopeRowInP$ $[\vFSSlopeRowInPLo,\vFSSlopeRowInPHi]$ for the restricted set (Figure~\ref{fig:theory}). At fixed $r$, the restricted deviation does not depend on $p$, as predicted. The measured unrestricted exponent crosses over from $\vFSSlopeFullInPSmallN$ at $n=\vFSNSmallN$ to $\vFSSlopeFullInPLargeN$ at $n=\vFSNLargeN$, consistent with the transition from $\dsr^{1}$ to $\dsr^{1/2}$ growth in Eq.~\eqref{eq:conc}.

\paragraph{(B$'$) Theorem~\ref{thm:untrunc}.}
Evaluated without truncation on the restricted search set, the fitted decay in $n$ is $\vTwoSlopeN$ $[\vTwoSlopeNLo,\vTwoSlopeNHi]$ and the dependence on $p$ is $\vTwoSlopeP$ $[\vTwoSlopePLo,\vTwoSlopePHi]$. The no-clipping condition holds throughout this grid, making the truncated and untruncated objectives pointwise identical here.

\paragraph{(C) Truncation certificate.}
The ratio $\max_i\|\Pi_{\mathcal S}z_i\|/T$ from Lemma~\ref{lem:noclip} never exceeds $\vFSNoClipRow{}$ (median $\vFSNoClipRowMed{}$) on the restricted search. The truncated and untruncated criteria are therefore identical on all $\vFSNoClipCells{}$ cells. On the full sphere, the ratio reaches $\vFSNoClipFull{}$, and the certificate fails for $p \ge \vFSNoClipFailP$.

Under the plug-in variance convention, $\hat\kappa_{\mathrm{impl}}(u)+3 = \big(\tfrac{n-1}{n}\big)^2 \big(\hat\kappa_T(u)+3\big)$. The minimizer and direction ordering remain identical, with numerical values differing by $O(1/n)$.

\begin{figure}[t]
\centering
\includegraphics[width=0.72\textwidth]{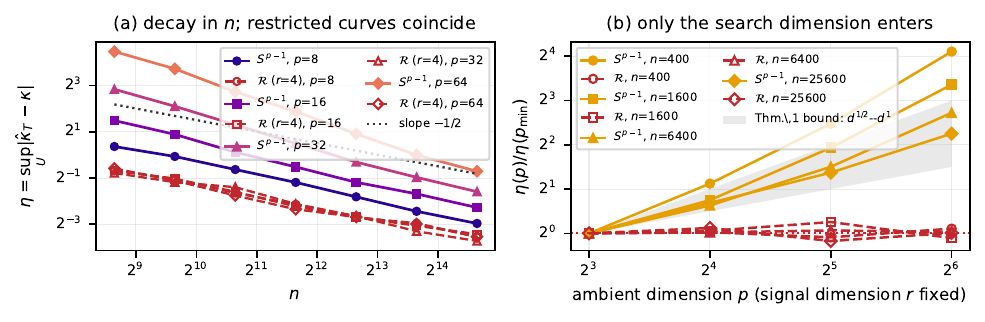}
\caption{Direct evaluation of Theorem~\ref{thm:conc}. (a) Uniform deviation $\eta=\sup_{u\in U}|\hat\kappa_T(u)-\kappa(u)|$ against $n$ for the unrestricted search at several $p$ (solid) and the operator-subspace restricted search at $r=4$ (dashed). (b) The same data against $p$ at fixed $n$; the shaded band spans the $\dsr^{1/2}$ to $\dsr^{1}$ dependence allowed by Eq.~\eqref{eq:conc}.}
\label{fig:theory}
\end{figure}

\section{Numerical checks for the lower bounds}
\label{app:lbchecks}
\paragraph{The two-point divergence.}
After whitening the observation law is $\tfrac12 N(+a,I)+\tfrac12 N(-a,I)$ and the divergence is exactly two-dimensional. Writing $\alpha=\|a\|$ and expanding $\log\cosh$ gives $\mathrm{KL}=\tfrac12\alpha^4\sin^2(2\theta_0)+O(\alpha^6)$, confirmed by quadrature ($\mathrm{KL}/\alpha^4\to\vkKLRatio$ against the predicted $\vkKLPredicted$). Since $\kappa_\star=-2\alpha^4/(1+\alpha^2)^2$, the divergence is $\Theta(|\kappa_\star|)$.

\label{app:lb}

\paragraph{(A) Proposition~\ref{prop:degen}.}
The witness in Eq.~\eqref{eq:degen} is constructed by solving $W_cb_0=\varepsilon$ for a balanced sign vector and normalizing. Over \vLBLOneCells{} configurations in the regime $r+2\le n\le p-r+1$ (even $n$, $p \in [64, 512]$), the witness reaches the floor $\hat\kappa=-2$ to $\vLBLOneFloorGap$, leaves $\vLBLOneLeak$ inside $\Rsp$, and strictly yields lower empirical kurtosis than $\hat\kappa(u_\star)$. Outside this regime, the construction leaves a median gap of $\vLBLOneGapOut$ to the floor.

\paragraph{(B) Depth law.}
The depth $D(m,n)=-\inf_b\hat\kappa(b)$ on pure $\NPDF(0,I_m)$ data is estimated by multi-start descent (\const{24} restarts $\times$ \const{400} iterations) over $m\in[4,252]$, $n\in[400,25600]$, and $\vLBDepthSeeds{}$ seeds.

Fitting $\log D=\log C+a\log m+b\log n$ over all $\vLBDepthCells{}$ measurements yields $a=\vLBDepthAllExpM$ and $b=\vLBDepthAllExpN$. Because $D$ is bounded by the hard floor $2$, $\vLBFitDropped{}$ measurements saturate above $\vLBFitThresh$. Refitting the $\vLBFitCells{}$ saturation-free measurements yields:
\begin{equation}
a=\vLBFitExpM\;[\vLBFitExpMLo,\vLBFitExpMHi],\qquad
b=\vLBFitExpN\;[\vLBFitExpNLo,\vLBFitExpNHi].
\label{eq:depthfit}
\end{equation}
The $n$-exponent is shallower than $-1/2$, matching the effective exponent $-1/2+1/(2\log n)$ from the $\sqrt{\log n}$ covering factor in Theorem~\ref{thm:conc}, which ranges from $\vLBLogCorrLo$ to $\vLBLogCorrHi$ over this grid. Pinning the exponents yields the floor-aware constant $D=\vLBFitC\sqrt{m/n}$ ($C\in[\vLBFitCLo,\vLBFitCHi]$). Equation~\eqref{eq:law} and the rescaled certificate cells use this constant rather than the biased all-cell value. Refitting with $\log\log m$ in place of $\log m$ gives $R^2=\vkDepthRsqLogLogM$ against $\vkDepthRsqLogM$, so the $m$-dependence is not identified over the tested range.

\paragraph{(C) Certificates.}
At each $(p,n)$, the constructed witness $b\in\Rsp^\perp$ is compared against the population optimum $\hat\kappa(u_\star)$ and the best empirical minimum found over $\Rsp$. The restricted reference uses \const{5}-start Adam checked against a dense evaluation of \const{400000} quasi-uniform directions polished locally (median gap $\vOptGapMed$, worst gap $\vOptGapMax$ over $\vOptCells{}$ cells).

Both searches use \const{24} restarts $\times$ \const{400} iterations. Below the boundary in \eqref{eq:law}, the certificate rate is $\vLBBandBelowRate{}$ over $\vLBBandBelowCells{}$ cells ($\vLBBandBelowDraws{}$ seed draws; exact 95\% interval $[\vLBBandBelowLo,\vLBBandBelowHi]$). Within a factor of two above the boundary it is $\vLBBandNearRate{}$ over $\vLBBandNearCells{}$ cells ($[\vLBBandNearLo,\vLBBandNearHi]$). Beyond that it is $\vLBBandAboveRate{}$ ($[\vLBBandAboveLo,\vLBBandAboveHi]$). Every seed yields a certificate up to $n/n_\times=\vLBBandBelowMax$, and none yields a certificate from $n/n_\times=\vLBBandAboveMin$.

\paragraph{(D) Remark~\ref{rem:trunc}.}
The largest-observation direction reaches $\hat\kappa=\vLBOutlierMax$ at $p=\vLBOutlierP$ against a population value of $0$. A single observation dictates the untruncated criterion at high $p$.

\section{Experimental details}
\label{app:details}
\paragraph{Controlled laboratory (LAB).}
$\Gt$ is drawn with orthonormal columns by QR factorization of a Gaussian matrix. Component means are placed on orthogonal latent axes at radius $\Delta/2$, and component covariances are $s_k^2I$. The projected separation, Bayes error, $\kappa$, and sub-Gaussian constant $\sigma$ are strictly independent of $p$ and $r$. All $\vLabCells{}$ factorial cells share identical population values. The population Bayes error of this configuration is $\vkBayesPop$ and $|\kappa_\star|=\vkKappaStar$.

\paragraph{Physical benchmark (PB).}
The physical benchmark uses a blur-type operator with $\mathrm{cond}(\Gt)\approx10^4$ and $\bK\ne0$. It uses $N=\const{3200}$ training points, $n_{\text{test}}=\const{2000}$ held out points, $\sigma_y=\const{0.05}$, $M=\const{2}$ regimes, and 5 seeds.

\paragraph{Discovery threshold.}
The discovery threshold is $n^\star(\epsilon,\delta) = \min\{n:\Pr_{\text{seed}}(\text{excess}\le\epsilon)\ge1-\delta\}$ on a $\log_2$ grid. The primary setting is $\epsilon=\const{0.01}$ and $\delta=\const{0.2}$.

\paragraph{Model comparison.}
Models are compared by BIC with bootstrap intervals on the exponents and residual inspection. The candidate $n^\star\sim r/\kappa^2$ is omitted from the factorial because $\kappa$ is fixed by construction, making $-\log\kappa^2$ collinear with the intercept.

\paragraph{Label-free permutation.}
Learned components are matched to true regimes by the Hungarian algorithm using $\|\Gt\hat\mu_j-\Gt\mu_k\|^2$ on the training split only. Paired comparisons use exact McNemar tests with discordant counts pooled over seeds.

\paragraph{Appended-noise control.}
The base problem uses $p_0=\const{16}$, $r=\const{4}$, $\Delta=\const{1.6}$, $s=\const{0.5}$, and $q\in\{0,16,48,112,240\}$ appended $\NPDF(0,1)$ coordinates. The appended block is drawn once per seed and truncated, leaving the first $p_0$ coordinates and exact Bayes decisions identical across $q$.

\paragraph{Subspace misspecification.}
The base configuration is $p=\const{32}, r=\const{4}$ over 10 seeds. The perturbed operator is $\Gt+\rho E$ for a rescaled Gaussian $E$ satisfying $\|\rho E\|_F=\rho\|\Gt\|_F$, evaluated at $\rho\in\{0.02,0.05,0.1,0.2,0.4,0.8,1.6\}$. The data-driven variant uses the top-$r$ eigenvectors of the training empirical covariance.

\paragraph{Nonlinear benchmark.}
The problem $-\Delta u+\kappa u^3=a$ is evaluated on a $21\times21$ grid with $\kappa=\const{40}$. The posterior uses a per-regime linearization as a self-normalized importance proposal.

\section{Information used by each method}
\label{app:info}
The known forward model (operator $\bH$, decoder basis $\bPhi$, sensor noise $\sigma$, and field covariance $\bK$) provides the whitened geometry. The true generative parameters $(\pi,\mu,\bSig)$ are supplied only to the exact Bayes reference oracle. The test labels are withheld from all methods during fitting and component matching. The full-space/restricted comparison isolates operator knowledge, as $\Rsp=\mathrm{col}(\Gt)$ is a deterministic function of the known forward model.

\begin{table}[ht]
\centering
\small
\begin{tabular}{lccccc}
\toprule
 & fwd.\ model & regimes & subspace & true & \\
method & $\bH,\bPhi,\sigma,\bK$ & $M$ & $\Rsp$ & $(\pi,\mu,\bSig)$ & labels \\
\midrule
oracle (control) & yes & yes & yes & \textbf{yes} & no \\
latent-space $k$-means & yes & yes & not used & no & no \\
kurtosis, full space & yes & yes & not used & no & no \\
kurtosis, operator-subspace & yes & yes & \textbf{used} & no & no \\
kurtosis, PCA-restricted & yes & yes & \textbf{estimated} & no & no \\
kurtosis, PCA, $\hat r$ from MP & \textbf{no} & yes & \textbf{estimated} & no & no \\
second-moment spectral & yes & yes & not used & no & no \\
multi-start EM / annealing & yes & yes & not used & no & no \\
\bottomrule
\end{tabular}
\caption{Train-time information available to each method.}
\label{tab:info}
\end{table}

\section{Optimization budget}
\label{app:budget}
The kurtosis search uses Adam on the unit sphere with a fixed budget: $\vProtNstarRestarts{}$ random starts $\times$ \const{250} iterations for both full-space and operator-subspace restricted searches at every $p$. The number of objective evaluations is matched rather than wall-clock time.

Varying the random starts over $\{\vBudLo,24,\vBudHi\}$ at $p\in\{6,12,24,48\}$, $r=2$, and 8 seeds yields identical discovery thresholds across budgets. The mean angular change from the smallest to the largest budget is $\vBudFullDelta^\circ$ for the unrestricted search (Wilcoxon $p=\vBudFullP$) and $\vBudRowDelta^\circ$ for the restricted search ($p=\vBudRowP$). The restart count has no systematic effect on angular error over this $\vBudRatio\times$ range.

\section{Reproducibility summary}
\label{app:repro}
\textbf{Split:} One draw per seed. All methods are evaluated on the identical held-out set, with component permutations fitted exclusively on the training split.\\
\textbf{Optimizer:} Adam, $\mathrm{lr}=\const{0.08}$, \const{250} iterations, retaining the lowest empirical kurtosis without early stopping. EM uses \const{25}--\const{40} epochs.\\
\textbf{Statistics:} Exact McNemar tests with seed-pooled discordant counts, \const{2000}-resample seed-level bootstrap intervals, and profile-likelihood intervals for censored fits.\\
\textbf{Compute:} Executed on CPU in double precision.

\section{Algorithm}
\label{app:algo}
The projection pursuit objective $\hat\kappa_T(u)$ is minimized over $u \in \mathcal{S} \cap S^{p-1}$ using multi-start projected Adam. The subspace $\mathcal{S}$ is $\Rb^p$ for the unrestricted search, $\Rsp$ for the operator-restricted search, and the span of the top $\hat{r}$ eigenvectors of the empirical covariance for the PCA-restricted search.

For each initialization, the update rule at iteration $t$ is:
\begin{align*}
    g_t &= \nabla_u \hat\kappa_T(u_t) \\
    \tilde{u}_{t+1} &= \mathrm{AdamUpdate}(u_t, g_t, \alpha) \\
    u_{t+1} &= \frac{\Pi_{\mathcal{S}} \tilde{u}_{t+1}}{\|\Pi_{\mathcal{S}} \tilde{u}_{t+1}\|_2}
\end{align*}
Gradients are computed via automatic differentiation. The learning rate is $\alpha=\const{0.08}$ with default momentum parameters $\beta_1=\const{0.9}, \beta_2=\const{0.999}$. Each restart runs for \const{250} iterations. The returned direction $\hat{u}$ is the one achieving the minimum $\hat\kappa_T$ across all $\vProtNstarRestarts{}$ restarts.

\section{Baselines}
\label{app:baselines}
Four of the methods listed in Table~\ref{tab:info} serve as baselines:
\begin{enumerate}
    \item \textbf{Latent-space $k$-means:} Applied directly to the exact latent coordinates $z$. This provides an omniscient performance ceiling, as it bypasses the forward operator entirely.
    \item \textbf{Second-moment spectral:} Computes the top eigenvector of the empirical covariance matrix of the whitened observations $\yt$. It recovers the span of the component means only when the signal covariance spike separates from the Marchenko--Pastur bulk. That spike is
\[
\varsigma^2\left(s^2+\Delta^2/4\right),
\]
so by the Baik--Ben Arous--P\'ech\'e threshold \citep{baik2004phasetransitionlargesteigenvalue} the requirement is $n\gtrsim p/\varsigma^4$. This is the same condition that sets $\varsigma_{\mathrm{PCA}}$ in Section~\ref{sec:law}.($\varsigma > \varsigma_{\mathrm{PCA}}$).
    \item \textbf{Gaussian Mixture Model (EM):} Exact EM for the latent mixture, operator held fixed, run on the whitened observations. The E-step is taken in the $r$-dimensional latent space rather than the ambient space, using the Woodbury identity, so each sweep costs $O(nr^2)$ and not $O(np^2)$; component covariances are full, not diagonal. It is initialized once from the candidate direction under test and run for at most \const{40} sweeps, stopping early when all parameters move by less than $\const{10^{-9}}$.
    \item \textbf{Exact Bayes oracle:} Computes the posterior regime probabilities using the exact generative parameters $(\pi, \mu, \bSig)$.
\end{enumerate}

\section{Extended related work}
\label{app:extended_rw}
\paragraph{Spiked covariance models.}
The threshold $\varsigma_{\mathrm{PCA}}$ is governed by the BBP phase transition \citep{baik2004phasetransitionlargesteigenvalue}. For a rank-$1$ deformation of a sample covariance matrix of aspect ratio $\gamma = p/n$, the principal eigenvector correlates with the signal direction only when the signal-to-noise ratio exceeds $\sqrt{\gamma}$. In our setting, the operator subspace restriction bypasses this transition, replacing the dependence on $p/n$ with $r/n$ or $\dsr/n$.

\paragraph{Higher-order tensors.}
Methods utilizing full fourth-order cumulant tensors \citep{anandkumar2014tensor} guarantee consistent recovery of latent components under mild non-degeneracy conditions. However, forming the full tensor requires $O(p^4)$ memory and $O(np^4)$ operations, which is intractable for the ambient dimensions evaluated here. Unlabeled projection pursuit compresses this search to $O(np)$ per iteration by optimizing a one-dimensional projection index directly on the sphere.

\section{Extended sweeps}
\label{app:sweeps}

\paragraph{Varying the angular success criterion.}
The thresholds are recomputed at success criteria $\theta\in\{\const{10}^\circ,\const{15}^\circ,\const{20}^\circ\}$ from the per-cell angles already recorded, requiring no additional simulation. The fitted exponents are $\vkExpCritTen$, $\vkExpCritFifteen$ and $\vkExpCritTwenty$, each inside the bootstrap interval $[\vkLawExpLo,\vkLawExpHi]$, while the crossing moves from $\vkCrossTen$ to $\vkCrossFifteen$ to $\vkCrossTwenty$ (Table~\ref{tab:crossing_shifts}).

\paragraph{Anisotropic latent covariance.}
The latent covariance is set to $\bSig = s^2\,\mathrm{diag}(1, a^2, \ldots, a^2)$ with $a\in\{\const{1},\const{10}\}$. As reported in Section~\ref{sec:notaconstant}, $a>1$ distorts the search-space geometry: $\varsigma_{\mathrm{PCA}}$ is invariant to four significant figures, while the threshold ratio $\varsigma_{\mathrm{KPP}}/\varsigma_{\mathrm{PCA}}$ shifts upward at $a=\const{10}$. An earlier version of this experiment measured that shift at three $(p,n)$ points, all at $r=\const{4}$, and found it constant at \const{1.10}. Since the collapse variable is $rn/p^2$, varying $n/p^2$ at fixed rank does not test constancy in $r$. On a grid of $\vkAnisoCells$ cells spanning $r\in\{\const{4},\const{16}\}$ the shift has median $\vkAnisoShiftMed$ over $[\vkAnisoShiftLo,\vkAnisoShiftHi]$, equal to $\vkAnisoShiftRFour$ at $r=\const{4}$ and $\vkAnisoShiftRSixteen$ at $r=\const{16}$; the difference in medians between ranks is $+\const{0.51}$ with a bootstrap interval $[\const{-0.14},\const{+0.73}]$, so four matched pairs per rank establish that the shift is not demonstrably constant without establishing what it depends on. Fitting the law separately on this grid gives $\vkAnisoCalExp\,[\vkAnisoCalExpLo,\vkAnisoCalExpHi]$ with $R^2=\vkAnisoCalRTwo$: the collapse survives anisotropy, but eight cells per arm do not constrain the exponent, and we quote the main grid's throughout. The thresholds in this sweep use a criterion set at \const{0.3} times the null median angle for each configuration, rather than a fixed angle or a common false-positive rate, for the reason given in Section~\ref{sec:calibration}.

\paragraph{Calibrating an angular criterion: detail.}
Over \const{40} no-signal seeds, the fifth percentile of the angle is $\vkNullIsoLo^\circ$--$\vkNullIsoHi^\circ$ at $a=\const{1}$ and $\vkNullAniLo^\circ$--$\vkNullAniHi^\circ$ at $a=\const{10}$. A criterion calibrated to a common false-positive rate is approximately $\vkNullRatio$ times stricter in the anisotropic case. The \const{15}$^\circ$ criterion used on the main grid ($a=\const{1}$, null fifth percentile $\vkNullIsoLo^\circ$--$\vkNullIsoHi^\circ$) is conservative. The $\vkNullRatio$-fold difference in calibrated criteria between two configurations of the same model prevents the use of a transportable crossing constant.

\end{document}